\documentclass{article}

\usepackage[preprint]{neurips_2026}

\usepackage[utf8]{inputenc} 
\usepackage[T1]{fontenc}    
\usepackage{hyperref}       
\usepackage{url}            
\usepackage{booktabs}       
\usepackage{amsfonts}       
\usepackage{nicefrac}       
\usepackage{microtype}      
\usepackage{xcolor}         

\usepackage{bm}
\usepackage{graphicx}
\usepackage[table]{xcolor}
\usepackage{enumitem}

\usepackage{arydshln}
\usepackage{amsmath}
\usepackage{amssymb}
\usepackage{amsthm}
\usepackage{algorithm}
\usepackage{algpseudocode}
\usepackage{bbm}

\theoremstyle{plain}
\newtheorem{theorem}{Theorem}[section]
\newtheorem{proposition}[theorem]{Proposition}
\newtheorem{lemma}[theorem]{Lemma}

\theoremstyle{definition}
\newtheorem{definition}[theorem]{Definition}

\theoremstyle{remark}

\newcommand{\SDRL}{\textsc{Sdrl}}

\usepackage{wrapfig}
\usepackage{caption}
\algrenewcommand\algorithmicrequire{\textbf{Input:}}
\algrenewcommand\algorithmicensure{\textbf{Output:}}

\title{Learning from Self-Debate: Preparing Reasoning Models for Multi-Agent Debate}

\author{%
  Chenxi Liu, Yanshuo Chen, Ruibo Chen, Tianyi Xiong, Tong Zheng, Heng Huang \\
  Department of Computer Science\\
  University of Maryland, College Park\\
}

\begin{document}

\maketitle

\begin{abstract}
The reasoning abilities of large language models (LLMs) have been substantially improved by reinforcement learning with verifiable rewards (RLVR). At test time, collaborative reasoning through Multi-Agent Debate (MAD) has emerged as a promising approach for enhancing LLM performance. However, current RLVR methods typically train LLMs to solve problems in isolation, without explicitly preparing them to synthesize and benefit from different rationales that arise during debate. In this work, we propose \emph{Self-Debate Reinforcement Learning} (SDRL), a training framework where models learn from self-debate, equipping a single LLM with both strong standalone problem-solving ability and the capability to process diverse reasoning trajectories in MAD. Given a prompt, SDRL first samples multiple candidate solutions, then constructs a debate context with diverse reasoning paths and generates second-turn responses conditioned on this context. Finally, SDRL jointly optimizes both the initial and debate-conditioned responses, yielding a model that is effective as both a standalone solver and a debate participant. Experiments across multiple base models and reasoning benchmarks show that SDRL consistently improves MAD performance across diverse debate protocols and agent configurations, while simultaneously strengthening single-model reasoning. Code is available at: \href{https://github.com/DawnLIU35/SDRL}{\texttt{github.com/DawnLIU35/SDRL}}.
\end{abstract}

\section{Introduction}
\label{sec:introduction}

Large Language Models (LLMs) have recently demonstrated strong reasoning capabilities, achieving impressive performance on tasks such as mathematics, programming, and tool use~\cite{guo2025deepseek,yang2025qwen3,liu2025deepseek}. Reinforcement Learning with Verifiable Rewards (RLVR) plays a crucial role in enabling these capabilities by leveraging tasks with automatically checkable outcomes to provide reliable supervision for scaling reasoning. To further enhance LLM performance, a growing line of work~\cite{wang2022self, xiong2025llava,wang2025llava,zheng2025parallel} studies test-time scaling strategies to further improve the reasoning ability of LLMs. Among these approaches, collaborative reasoning~\cite{du2023improving,liang2024encouraging} through Multi-Agent Debate (MAD) has emerged as a promising paradigm: multiple LLM agents propose solutions to a shared question and iteratively update their answers in response to their peers~\cite{choi2025debate,yi2025debate}. 

Despite this promise, effective methodologies for preparing reasoning models for MAD interaction remain underexplored. Existing RLVR methods typically train LLMs to reason and then solve problems within single trajectories, failing to explicitly prepare them for collaborative environments where they must process diverse rationales. This mismatch between training and inference can limit LLM effectiveness in MAD settings~\cite{choi2025debate,kumaran2025overconfidence,li2025off}. Furthermore, recent attempts to train debate behaviors typically optimize the entire multi-agent systems, significantly increasing both training costs and deployment complexity~\cite{liao2025marft,liu2025llm,marti2025,li2025advancing}. These limitations highlight the need for a framework that produces a \emph{single} LLM capable of both independent problem-solving and productive collaboration in MAD.

In this work, we develop a theoretical analysis that clarifies why debate can help and when its benefits saturate. Inspired by the Dirichlet Compound Multinomial (DCM) framework of \citet{choi2025debate}, we model multi-agent debate as Bayesian belief updating and recover their key implication that standard MAD induces a martingale over each agent’s belief in the correct answer. We then extend this framework by explicitly disentangling the contributions of majority voting and agents’ \emph{private critique} in MAD, and show that improving agents’ private critique can improve overall MAD performance. Building on this insight, we propose \emph{Self-Debate Reinforcement Learning} (\SDRL{}), a learning paradigm that enables a single model to learn from its own reasoning conflicts, effectively preparing it for both standalone problem-solving and multi-agent interactions. Given a prompt, \SDRL{} first samples multiple candidate solutions from the current policy and evaluates them using verifiable rewards. \SDRL{} then constructs debate pairs by selecting different candidates from these initial solutions, forcing the model to confront divergent reasoning trajectories. Conditioned on a debate pair, the model generates second-round responses that deliberate over both candidates to produce a final answer. Finally, \SDRL{} jointly optimizes both the initial and second-round responses under verifiable rewards. This joint objective encourages the model to produce accurate first-pass solutions while also learning when and how to revise its reasoning during debate-style interaction, enabling a single trained model to serve as a highly competent agent within MAD.

Overall, our \textbf{contributions} can be summarized as follows:
\begin{enumerate}[itemsep=6pt, parsep=0pt, topsep=1pt]
    \item We introduce a theoretical framework that disentangles the contributions of debate and majority voting in the MAD paradigm, and we identify improving agents’ private critique as a key driver of debate gains.
    \item We propose \emph{Self-Debate Reinforcement Learning} (\SDRL{}), a training framework that equips a \emph{single} reasoning model with both independent problem-solving ability and collaborative debate skills by jointly optimizing both objectives.
    \item Extensive experiments across multiple reasoning benchmarks demonstrate that \SDRL{} consistently enhances multi-agent debate performance while simultaneously strengthening single-agent accuracy.
\end{enumerate}

\section{Preliminaries}
\label{preliminaries}

\paragraph{Multi-Agent Debate.}
We formulate the debate process as an iterative consensus-seeking mechanism among a set of $N$ language model agents, denoted by $\Pi = \{\pi_1, \dots, \pi_N\}$. Let $\mathcal{Q}$ represent the input space and $\mathcal{O}$ the output space. Each agent $\pi_i$ generates an independent response $o_{i,0} \sim \pi_i(q)$ for a given input $q \in \mathcal{Q}$.

In the \textit{Majority Voting} setting~\cite{zhao2025surveylargelanguagemodels}, these independent outputs are immediately aggregated via a voting function $\mathcal{V}: \mathcal{O}^N \rightarrow \mathcal{O}$ to produce the final prediction:
\begin{equation}
    o_0 = \mathcal{V}(o_{1,0}, \dots, o_{N,0}).
\end{equation}
Multi-Agent Debate (MAD) extends this static ensemble approach by introducing a $T$ round communication framework. Under the simultaneous-talk protocol~\cite{chan2023chateval,choi2025debate}, agents update their beliefs in parallel over discrete rounds $t \in \{1, \dots, T\}$. At step $t$, agent $\pi_i$ conditions its update on the set of peer responses generated in the previous round, formally defined as:
\begin{equation}
    \mathcal{O}_i^{(t-1)} = \{ o_{j,t-1} \mid j \in \mathcal{N}(i) \}.
\end{equation}
where $\mathcal{N}(i) \subseteq \{1, \dots, N\}$ denotes the neighborhood observable to agent $\pi_i$, including itself. The agent then refines its answer using a debate protocol $\mathcal{D}$, which integrates the original query $q$ and the peer context $\mathcal{O}_i^{(t-1)}$ to produce the updated response:
\begin{equation}
    o_{i,t} = \mathcal{D}\left(q; \mathcal{O}_i^{(t-1)}\right).
\end{equation}
This iterative refinement can be viewed as a functional composition over $T$ rounds. The state of agent $\pi_i$ at the final round is expressed as:
\begin{equation}
    o_{i,T} = (\mathcal{D} \circ \dots \circ \mathcal{D})(q; \mathcal{O}_i) = \mathcal{D}^{(T)}(q; \mathcal{O}_i).
\end{equation}
Finally, the system-level output is derived by aggregating the refined responses from the terminal state:
\begin{equation}
    o_T = \mathcal{V}(o_{1,T}, \dots, o_{N,T}).
\end{equation}
Following previous works~\cite{choi2025debate,yi2025debate}, we focus on homogeneous agent settings, where all $\pi_i$ share identical architectures.

\paragraph{Reinforcement Learning.}
To eschew the computational overhead of a separate critic network, we adopt the Group Relative Policy Optimization (GRPO) framework~\citep{shao2024deepseekmath,guo2025deepseek}. For a specific query $q$, the behavior policy $\pi_{\theta_\text{old}}$ samples a group of $G$ responses $\{o_i\}_{i=1}^G$. The advantage for the $i$-th response is estimated by normalizing the reward $r_i$ against the group statistics:
\begin{equation}
\label{eq:advantage}
\hat{A}_{i,t} = \frac{r_i - \text{mean}(\{r_i\}_{i=1}^G)}{\text{std}(\{r_i\}_{i=1}^G)}.
\end{equation}
Building on this, we employ Decoupled Clip and Dynamic sAmpling Policy Optimization (DAPO)~\citep{yu2025dapo} to stabilize updates for long chain-of-thought reasoning. DAPO mitigates entropy collapse and reward noise through asymmetric clipping and dynamic sampling constraints. The policy is optimized via the following token-level gradient objective:
{\small
\begin{equation}
\begin{aligned}
\mathcal{J}_{\text{DAPO}}(\theta) =\quad& \mathbb{E}_{(q,o^*)\sim \mathcal{Q}, \{o_i\}_{i=1}^G\sim \pi_{\theta_\text{old}}(\cdot\mid q)}
\Bigg[\frac{1}{\sum_{i=1}^{G}|o_i|}\sum_{i=1}^{G}\sum_{t=1}^{|o_i|} \\
&\hspace{-4em} \min \Big( r_{i,t}(\theta) \hat{A}_{i,t},  
\ \text{clip} \Big( r_{i,t}(\theta), 1 - {\varepsilon_{\text{low}}}, 1 + {\varepsilon_{\text{high}}} \Big) \hat{A}_{i,t} \Big) \Bigg]
\\
\text{s.t.}\quad& 0< \Big|\{o_i\mid\texttt{is\_equivalent}(o^*,o_i)\}\Big|< G,
\label{eq:dapoloss}
\end{aligned}
\end{equation}
}
where $r_{i,t}(\theta)=\frac{\pi_{\theta}(o_{i,t} \mid q, o_{i,<t})}{\pi_{\theta_{\text{old}}}(o_{i,t} \mid q,o_{i,<t})}$ is the importance ratio, and $o^*$ denotes the ground-truth answer.

\section{Theoretical Analysis}
\label{theoretical}

\paragraph{Motivation.}
Under the Dirichlet Compound Multinomial (DCM) model in \cite{choi2025debate}, the standard Bayesian debate induces a martingale over each agent's belief in the correct answer, suggesting that debate alone does not improve expected correctness.
A key gap is that real agents also apply \emph{private critique} after reading peers, and update their output based on their observation and critique~\cite{gou2023critic}.
We disentangle this effect with majority voting and show that
debate training improves multi-round debate by making this \emph{private critique} positively aligned with the ground truth,
thereby inducing a positive drift that breaks martingale neutrality.

\textbf{Setup.}
Fix an input question \(q\) and a finite answer set \(\mathcal{A}=\{1,\dots,K\}\). Without loss of generality, we set answer \(1\) as the correct one.
Following the DCM model~\cite{choi2025debate}, at round \(t\), agent \(i\) maintains Dirichlet parameters \(\alpha_{i,t}\in\mathbb{R}^K_{+}\) and observes neighbors \(N(i)\). Let $a(\cdot)$ extract the final answer choice from a response. We write $y_{i,t} := a(o_{i,t}) \in \mathcal{A}$ for the discrete answer label used in the DCM analysis.
We denote $L^1$-norm as \(\|v\|_1=\sum_{k=1}^K v^{(k)}\) and define the Dirichlet mean as \(\bar\theta_{i,t}:=\alpha_{i,t}/\|\alpha_{i,t}\|_1\in\Delta^{K}\).

\begin{definition}[Critique-augmented belief update]
\label{def:critique_update}
Let \(c_{i,t}\in\mathbb{N}^K\) be the neighbor count vector at round \(t\), 
\(c^{(k)}_{i,t}=\sum_{j\in N(i)}\mathbbm{1}\{y_{j,t-1}=k\}\).
In addition to neighborhood evidence \(c_{i,t}\), after observing the round-$t-1$ debate context, agent $i$ computes $\beta_{i,t-1}$
representing its \emph{private critique} from the debate context.
The belief update is
\begin{equation}
\alpha_{i,t} \;=\; \alpha_{i,t-1} \;+\; \beta_{i,t-1} \;+\; w_i\, c_{i,t},
\label{eq:update}
\end{equation}
where \(w_i\ge 0\) controls the strength of the social signal.
\end{definition}

\begin{definition}[Response generation via DCM]
\label{def:dcm}
Given \(\alpha_{i,t}\), the agent generates \(\theta_{i,t}\sim\mathrm{Dirichlet}(\alpha_{i,t})\)
and \(y_{i,t}\sim\mathrm{Categorical}(\theta_{i,t})\).
Marginally, \(\Pr(y_{i,t}=k\mid \alpha_{i,t})=\bar\theta^{(k)}_{i,t}\).
\end{definition}

Then the belief in the correct answer can be defined via Dirichlet mean:
\begin{equation}
p_{i,t} \;:=\; \bar\theta^{(1)}_{i,t} \;=\; \frac{\alpha^{(1)}_{i,t}}{\|\alpha_{i,t}\|_1}.
\label{eq:pt}
\end{equation}
We next let $\mathcal{F}_t$ denote the $\sigma$-algebra of information available at the start of round $t$ (see Appendix \ref{filtration} for the explicit definition).

\begin{definition}[Advantage of private critique]
\label{def:adv}
Assume \(\|\beta_{i,t}\|_1=m_\beta\) is constant (see Appendix \ref{beta_explained} for explanation).
The \emph{critique advantage} is
\begin{equation}
\delta_{i,t}\;:=\;\mathbb{E}\!\left[\beta^{(1)}_{i,t}\mid\mathcal{F}_{t}\right] \;-\; m_\beta\, p_{i,t}.
\label{eq:adv}
\end{equation}
Thus \(\delta_{i,t}>0\) means the critique allocates more belief to the correct answer in this round than a previous round's belief proportional to \(p_{i,t}\).
\end{definition}

\begin{theorem}[Critique induces belief drift]
\label{thm:drift}
Assume the mean-consistency condition \(\bar p_{N(i),t-1}=p_{i,t-1}\)
(the same condition under which standard MAD is a martingale in \cite{choi2025debate}). Let \(C_i:=m_\beta+w_i|N(i)|\), 
then
\begin{equation}
\mathbb{E}\!\left[p_{i,t}\mid \mathcal{F}_{t-1}\right]
=
p_{i,t-1}
+\frac{\delta_{i,t-1}}{\|\alpha_{i,t-1}\|_1+C_i}.
\label{eq:drift}
\end{equation}
\end{theorem}

\begin{proof}[Proof sketch]
Expand \(p_{i,t}\) from \eqref{eq:update} and take conditional expectation.
Use \(\mathbb{E}[c^{(1)}_{i,t}\mid\mathcal{F}_{t-1}]=\sum_{j\in N(i)}p_{j,t-1}=|N(i)|\bar p_{N(i),t-1}\)
and \(\alpha^{(1)}_{i,t-1}=p_{i,t-1}\|\alpha_{i,t-1}\|_1\) to obtain Lemma~\ref{lem:drift_decomp}.
Under mean-consistency assumption, the neighborhood drift cancels, yielding \eqref{eq:drift}.
\end{proof}

\begin{lemma}[Accumulated drift and diminishing returns]
\label{lem:accum_drift}
Assume mean-consistency and \(\delta_{i,t}\ge \mu\) for \(t=0,\dots,T-1\).
Let \(S_{i,0}:=\|\alpha_{i,0}\|_1\),
then
\begin{equation}
\begin{aligned}
\mathbb{E}[p_{i,T}]
&\ge p_{i,0}
+ \mu\sum_{t=1}^{T}\frac{1}{S_{i,0}+tC_i}
\ge p_{i,0}
+ \frac{\mu}{C_i}\log\!\left(\frac{S_{i,0}+(T+1)C_i}{S_{i,0}+C_i}\right).
\end{aligned}
\label{eq:accum_drift_main}
\end{equation}
\end{lemma}

\begin{proposition}[Training increases critique advantage]
\label{prop:training_adv}
Suppose debate training yields \(\delta_{i,t}\ge \mu>0\) on the evaluation-time debate distribution for early rounds.
Then Theorem~\ref{thm:drift} and Lemma~\ref{lem:accum_drift} imply \(\mathbb{E}[p_{i,t}]\) increases with \(t\).
\end{proposition}

\begin{lemma}[Correlation shrinks the effective ensemble size]
\label{lem:corr_neff}
Consider a single debate round with $N$ agents producing answers
$Y_1,\dots,Y_N \in \{1,\dots,K\}$.
Assume each $Y_n$ has the same marginal distribution $p\in\Delta^K$
with gap $\Delta := p_1-p_2>0$ (\cite{choi2025debate} Theorem 1).
Let $\hat p$ be the empirical distribution and
$y_{\mathrm{mv}}=\arg\max_k \hat p_k$ be the majority vote.

Let the correlation parameter be
\[
\rho := \max_{k\in[K]}\ \max_{a\neq b}\ 
\frac{\mathrm{Cov}(\mathbbm{1}\{Y_a=k\},\mathbbm{1}\{Y_b=k\})}{p_k(1-p_k)} \in [0,1].
\]
with the assumption that $p_k \in (0,1)$ for all $k$.
Then the plurality-vote error probability admits the bound
\[
\begin{aligned}
\Pr(y_{\mathrm{mv}}\neq 1)
&\le \frac{K(1+(N-1)\rho)}{N\,\Delta^2}
= \frac{K}{N_{\mathrm{eff}}\Delta^2},\\
N_{\mathrm{eff}} &:= \frac{N}{1+(N-1)\rho}.
\end{aligned}
\]
\end{lemma}

\paragraph{Role of majority vs.\ private critique.}
Our analysis disentangles the contributions of \emph{majority voting} and \emph{private critique} in multi-round debate.
When there is no training signal from critique (\emph{i.e.,} $\delta_{i,t}=0$), Theorem~\ref{thm:drift} reduces to a neutral (martingale) belief evolution, and the end-to-end accuracy is governed purely by vote amplification—recovering the standard majority-voting perspective (\emph{e.g.,} \cite{choi2025debate}).
\SDRL{} helps by increasing the critique advantage $\delta$, which induces positive drift in each agent’s correctness; however, Lemma~\ref{lem:accum_drift} shows that even under sustained advantage ($\delta_{i,t}\ge \mu$), the improvement accumulates only logarithmically in the number of rounds and yields diminishing per-round gains.
At the same time, as rounds progress agents condition on increasingly similar contexts, raising answer correlation and effectively reducing the ensemble size (Lemma~\ref{lem:corr_neff}), which weakens majority-vote amplification.
Together, these two effects explain a common empirical pattern: performance improves in early rounds (positive critique drift with low correlation) but peaks quickly and can decline later as marginal drift shrinks and homogeneity erodes the benefits of voting.

\section{Methodology}
\label{methodology}

We propose Self-Debate Reinforcement Learning (\SDRL{}), an online reinforcement learning framework that trains a \emph{single} policy to improve both (i) single-agent reasoning and (ii) \emph{private critique} required by Multi-Agent Debate (MAD). Section~\ref{subsec:single_agent_reasoning} introduces how \SDRL{} uses verifiable rewards to train the model for single-agent reasoning. Section~\ref{subsec:debate_ability} introduces how \SDRL{} equips a single model with the private critique ability to discriminate between different opinions during debate interactions.

\subsection{Single-agent reasoning}
\label{subsec:single_agent_reasoning}

To improve the model’s reasoning ability in the single-agent setting, \SDRL{} follows the GRPO framework~\cite{shao2024deepseekmath} and optimizes the policy using verifiable rewards. For each prompt $q$, \SDRL{} samples $n$ independent responses  $\{o_i\}_{i=1}^n \sim \pi_\theta(\cdot \mid q)$ from the current policy and assigns a sparse outcome reward $r_j \in \{+1,-1\}$ based on the correctness of the final answer. We then compute advantages for each response using Eq.~\ref{eq:advantage}. To prepare candidate responses for debate training, we follow DAPO~\cite{yu2025dapo} to oversample prompts and filter out those whose responses yield zero advantages. This ensures that each prompt has responses corresponding to different final answers, capturing diverse opinions and reasoning trajectories.

\begin{algorithm}[t]
\footnotesize
\caption{Self-Debate Reinforcement Learning (\SDRL{})}
\label{alg:sdrl}
\vspace{-0.4em}
\begin{algorithmic}[1]
\Require Policy $\pi_\theta$, prompt set $\mathcal{D}=\{q_i\}_{i=1}^N$, initial rollouts $n$, debate rollouts $n_d$, training steps $K$
\Ensure Updated policy $\pi_{\theta_{\mathrm{updated}}}$

\For{$t=1,\dots,K$}
  \State Sample mini-batch $\mathbf{q}\subseteq\mathcal{D}$ and initialize $\mathbf{B}\leftarrow\emptyset$
  \For{each $q\in\mathbf{q}$}
    \State Sample initial rollouts $\{o_i\}_{i=1}^n\sim\pi_\theta(\cdot\mid q)$ and compute advantages $\{A_i\}_{i=1}^n$
    \State $\mathbf{B}\leftarrow \mathbf{B}\cup\{(q,\{o_i\}_{i=1}^n,\{A_i\}_{i=1}^n)\}$
    \State Select debate pair $(o^1,o^2)$ from $\{o_i\}_{i=1}^n$ and construct debate prompt $q_d$
    \State Sample debate rollouts $\{o_j^d\}_{j=1}^{n_d}\sim\pi_\theta(\cdot\mid q_d)$ and compute advantages $\{A_j^d\}_{j=1}^{n_d}$
    \State $\mathbf{B}\leftarrow \mathbf{B}\cup\{(q_d,\{o_j^d\}_{j=1}^{n_d},\{A_j^d\}_{j=1}^{n_d})\}$
  \EndFor
  \State Update $\pi_\theta$ using the RL optimizer on $\mathbf{B}$
\EndFor
\State \Return $\pi_{\theta_{\mathrm{updated}}}$
\end{algorithmic}
\vspace{-0.3em}
\end{algorithm}

\subsection{Debate training}
\label{subsec:debate_ability}

To improve the model's private critique ability, \SDRL{} constructs debate contexts that expose the policy to conflicting reasoning trajectories. Producing a correct final answer in these contexts requires the model to exercise private critique by distinguishing reliable reasoning from flawed trajectories. Given the initial rollout set for a prompt $q$, we construct a debate pair by selecting two responses from the candidate response pool, which is formed from the first-round responses generated for single-agent reasoning. Selecting two responses to form a debate pair is a computationally efficient design choice that induces the fundamental private critique behavior and generalizes to settings with more candidate responses.

\textbf{Debate pair construction.}
Let $\mathcal{O}(q)=\{o_i\}_{i=1}^n$ denote the initial rollouts, which we use as the candidate pool, and let $a(o_i)$ be the extracted final answer of rollout $o_i$.
\SDRL{} constructs a debate pair $\mathcal{P}(o^1, o^2)$ by selecting two candidates from $\mathcal{O}(q)$ using one of two pairing rules:
\begin{enumerate}
    \item \emph{Random pairing}. It samples two rollouts uniformly from $\mathcal{O}(q)$.
    This induces a broad distribution of pair diversity.
    Some pairs exhibit genuine disagreement, while others share the same final answer but differ in their reasoning trajectories, which can provide useful training signals for confirmation rather than unconditional revision.
    \item \emph{Frequency-based pairing}. It first identifies the most common answer $a_1$ and the second most common answer $a_2$ in $\{a(o_i)\}_{i=1}^n$.
    This rule more consistently exposes the model to the dominant competing beliefs expressed by its current policy, yielding sharper disagreement instances and therefore requiring stronger private critique ability.
\end{enumerate}
 We evaluate both constructions and use the same downstream prompting and optimization for either choice. 

\textbf{Debate prompt formulation.}
Given a selected pair $\mathcal{P}(o^1, o^2)$, we build a debate prompt $q_d$ by serializing the two candidate responses into a two-turn conversation, and then prompt the model to deliberate over both responses and produce a final answer. The detailed conversation format and prompt template are shown in the Appendix. For \emph{frequency-based pairing}, we randomize the order of $a_1$ and $a_2$ to prevent positional heuristics. Since the most common answer $a_1$ typically has a higher probability of being correct, always placing $a_1$ as the first response in $q_d$ could encourage the model to select the first-position answer rather than perform critical analysis of the underlying reasoning trajectories.

\section{Experiments}
\label{experiments}

\subsection{Implementation Details}
\label{sec:implementation}

\subsubsection{Training details}
 Following recent works~\citep{zheng2025first,cheng2025reasoning,zheng2025parallelr1parallelthinkingreinforcement}, We employ Qwen2.5-3B~\cite{qwen2025qwen25technicalreport}, Qwen3-4B-Base and Qwen3-8B-Base~\cite{yang2025qwen3} as our backbone models. Following works~\citep{yu2025dapo,cheng2025reasoning,cui2025entropy}, we utilize the DAPO-Math-17K dataset~\citep{yu2025dapo} for training. Given its demonstrated stability and superiority over vanilla GRPO~\citep{yu2025dapo,cheng2025reasoning}, we adopt the DAPO algorithm~\cite{yu2025dapo} as both our baseline and the core optimization method for our \SDRL{} approach. 

We configure the learning rate at $1 \times 10^{-6}$ with a linear warm-up over the first 10 rollout steps. The rollout phase utilizes a prompt batch size of $256$, generating $8$ responses per prompt. For each debate pair, we generate $4$ responses for Qwen2.5-3B and $8$ responses for Qwen3-4B-Base and Qwen3-8B-Base. We utilize a sparse reward signal, assigning $+1$ for correct solutions and $-1$ otherwise. All training experiments are implemented within the verl framework~\citep{sheng2024hybridflow}. Additional configuration details are provided in the Appendix.

\subsubsection{Evaluation}

\textbf{Benchmarks.} Following prior work~\cite{choi2025debate,zhao2025majority,li2025off}, we conduct evaluations on the mathematical reasoning benchmarks MATH500~\citep{hendrycks2021measuring}, AMC 2023, and AIME 2024/2025.

\textbf{Multi-Agent Debate.} Consistent with established MAD literature~\citep{choi2025debate}, we primarily evaluate under the \textit{Decentralized MAD} setting~\cite{du2023improving,choi2025debate}, where each agent observes all peer responses from the previous round. Additional ablation studies on more debate frameworks are provided in Section~\ref{further_analysis}. We benchmark performance against the strong \textit{Majority Voting} baseline, which selects the most common answer from the initial round without debate. Our main experiments utilize $N=5$ agents and $T=1$ debate round; ablations on $N$ and $T$ are provided separately. Inference is conducted with a rollout temperature of $1.0$ and top-$p$ sampling ($p = 0.9$). For all MAD experiments, we report the average score across $5$ independent runs.

\textbf{Single-Agent Reasoning.} Following recent RL reasoning literature~\citep{guo2025deepseek,liu2025explore}, we also assess model performance without debate. For each prompt, we sample $K$ independent responses, and report the average accuracy ($mean@K$) and majority-vote accuracy ($maj@K$), setting $K=32$ for the AMC and AIME datasets, and $K=4$ for the MATH500 benchmarks. Further evaluation details are available in the Appendix.

\begin{table*}[t]
  \caption{Multi-agent debate performance comparison of different model architectures between the DAPO baseline and \SDRL{} across math reasoning benchmarks. The debate system contains $5$ agents. \textit{Maj} denotes the majority-vote accuracy of the agents' direct responses to the question. \textit{Debate} denotes the performance of the decentralized multi-agent system after debate round $1$. $\Delta$ is the difference between \textit{Maj} and \textit{Debate}. Best results are bolded.}
  \label{tab:main-multi}
  \begin{center}
    \resizebox{\textwidth}{!}{
      \begin{tabular}{lccccccccccccccc}
        \toprule
        Method
          & \multicolumn{3}{c}{MATH500}
          & \multicolumn{3}{c}{AMC23}
          & \multicolumn{3}{c}{AIME24}
          & \multicolumn{3}{c}{AIME25}
          & \multicolumn{3}{c}{Avg.} \\
        \cmidrule(lr){2-4}\cmidrule(lr){5-7}\cmidrule(lr){8-10}\cmidrule(lr){11-13}\cmidrule(lr){14-16}
          & Maj & Debate & $\Delta$
          & Maj & Debate & $\Delta$
          & Maj & Debate & $\Delta$
          & Maj & Debate & $\Delta$
          & Maj & Debate & $\Delta$ \\
        \midrule
        \multicolumn{16}{c}{Qwen2.5-3B} \\
        \midrule
        \emph{DAPO}               & 70.9 & 71.1 & 0.2  & 49.0 & 50.0 & 1.0  & 7.3 & 10.0 & 2.7  & 4.0 & 3.3 & -0.7 & 32.8 & 33.6 & 0.8 \\
        + \SDRL{}-rand        & 70.7 & 71.4 & \textbf{0.7}  & 49.2 & \textbf{53.8} & \textbf{4.6}  & 8.7 & 11.7 & 3.0  & \textbf{5.6} & \textbf{6.1} & 0.5  & 33.6 & \textbf{35.8} & \textbf{2.2} \\
        + \SDRL{}-freq         & \textbf{71.4} & \textbf{71.7} & 0.3  & \textbf{52.5} & 53.5 & 1.0  & \textbf{9.3} & \textbf{13.3} & \textbf{4.0}  & 3.3 & 4.7 & \textbf{1.4} & \textbf{34.1} & \textbf{35.8} & 1.7 \\
        \midrule
        \multicolumn{16}{c}{Qwen3-4B-Base} \\
        \midrule
        \emph{DAPO}               & 86.9 & 83.1 & -3.8 & \textbf{73.0} & 76.0 & 3.0 & 26.7 & 28.3 & 1.6 & 24.7 & 24.0 & -0.7 & 52.8 & 52.9 & 0.1 \\
        + \SDRL{}-rand        & \textbf{88.8} & \textbf{86.6} & -2.2 & 72.5 & \textbf{80.0} & \textbf{7.5} & \textbf{30.8} & 31.7 & 0.9 & 22.7 & 24.0 & \textbf{1.3} & 53.7 & 55.6 & 1.9 \\
        + \SDRL{}-freq         & 87.0 & 85.9 & \textbf{-1.1} & 72.5 & 79.0 & 6.5 & 28.7 & \textbf{36.0} & \textbf{7.3} & \textbf{28.7} & \textbf{30.0} & \textbf{1.3} & \textbf{54.2} & \textbf{57.7} & \textbf{3.5} \\
        \midrule
        \multicolumn{16}{c}{Qwen3-8B-Base} \\
        \midrule
        \emph{DAPO}               & 86.6 & 85.3 & -1.3 & 79.0 & 77.5 & -1.5 & 28.0 & 26.7 & -1.3 & 23.3 & 24.7 & 1.4 & 54.2 & 53.6 & -0.6 \\
        + \SDRL{}-rand     & 86.2 & 86.5 & \textbf{0.3} & \textbf{81.5} & 86.5 & 5.0 & 31.0 & 34.8 & 3.8 & \textbf{26.7} & 32.0 & 5.3 & 56.4 & 60.0 & 3.6 \\
        + \SDRL{}-freq     & \textbf{88.0} & \textbf{88.2} & 0.2 & 80.0 & \textbf{88.0} & \textbf{8.0} & \textbf{33.3} & \textbf{37.9} & \textbf{4.6} & 25.6 & \textbf{32.2} & \textbf{6.6} & \textbf{56.7} & \textbf{61.6} & \textbf{4.9} \\
        \bottomrule
      \end{tabular}
    }
  \end{center}
  \vskip -0.1in
\end{table*}

\subsection{Multi-Agent Debate}
\label{debate_exp}
Table~\ref{tab:main-multi} compares \SDRL{} with the DAPO baseline under the decentralized MAD setting with $N{=}5$ agents and one debate round. \SDRL{}-rand denotes random pairing, and \SDRL{}-freq denotes frequency-based pairing for debate pair construction. The main observation is that \SDRL{} consistently improves multi-agent debate across model architectures. The strongest results are achieved by \SDRL{}-freq, which gives the best average \textit{Debate} accuracy for Qwen3-4B-Base and Qwen3-8B-Base, and matches the best result for Qwen2.5-3B.

For Qwen2.5-3B, \SDRL{} improves the post-debate performance from $33.6$ to $35.8$ on average. \SDRL{}-rand gives the largest average debate gain, increasing $\Delta$ from $0.8$ to $2.2$. \SDRL{}-freq reaches the same best average \textit{Debate} accuracy of $35.8$ and gives the strongest AIME24 result, improving \textit{Debate} from $10.0$ to $13.3$ and $\Delta$ from $2.7$ to $4.0$. These results show that \SDRL{} makes even the smaller model more effective in decentralized debate.

The advantage of \SDRL{}-freq becomes clearer on stronger backbones. For Qwen3-4B-Base, \SDRL{}-freq improves the average \textit{Maj} accuracy from $52.8$ to $54.2$ and the average \textit{Debate} accuracy from $52.9$ to $57.7$. It also increases the average debate gain from $0.1$ to $3.5$. The largest improvement appears on AIME24, where \textit{Debate} increases from $28.3$ to $36.0$. For Qwen3-8B-Base, \SDRL{}-freq further strengthens this trend. It improves the average \textit{Maj} accuracy from $54.2$ to $56.7$, raises the average \textit{Debate} accuracy from $53.6$ to $61.6$, and increases the average debate gain to $\Delta=4.9$. It also achieves the best post-debate result on most benchmark for this backbone.

Overall, these results demonstrate that \SDRL{} trains reasoning LLMs to better participate in multi-agent debate. The improvement appears in both the quality of the initial response pool and the effectiveness of the debate step. Among the two pairing strategies, \SDRL{}-freq provides the most reliable gains, especially on stronger models and challenging benchmarks. This suggests that frequency-based pairing creates more informative debate contexts, allowing agents to better compare competing solutions, revise erroneous trajectories, and convert debate into consistent performance gains.

\begin{table*}[t]
  \caption{Single-agent performance comparison of different model architectures between the DAPO baseline and \SDRL{} across several math reasoning benchmarks.  We report the average accuracy $\mathrm{mean}@K$ and the majority-vote accuracy $\mathrm{maj}@K$ over $K$ independent runs. Best results are bolded.}
  \label{tab:main-single}
  \begin{center}
    \begin{small}
        \setlength{\tabcolsep}{4.5pt}
        \resizebox{\textwidth}{!}{%
        \begin{tabular}{lcccccccccc}
          \toprule
          Method
            & \multicolumn{2}{c}{MATH500}
            & \multicolumn{2}{c}{AMC23}
            & \multicolumn{2}{c}{AIME24}
            & \multicolumn{2}{c}{AIME25}
            & \multicolumn{2}{c}{Avg.} \\
          \cmidrule(lr){2-3}\cmidrule(lr){4-5}\cmidrule(lr){6-7}\cmidrule(lr){8-9}\cmidrule(lr){10-11}
            & mean@4 & maj@4
            & mean@32 & maj@32
            & mean@32 & maj@32
            & mean@32  & maj@32
            & mean & maj \\
          \midrule
          \multicolumn{11}{c}{Qwen2.5-3B} \\
          \midrule
          \emph{DAPO}      & 64.5 & 70.4 & 48.5 & 56.8 & 7.2 & 13.5 & 3.1 & 4.2 & 30.8 & 36.2 \\
          + \SDRL{}-rand      & 65.1 & 70.6 & 48.1 & 57.4 & \textbf{9.9} & 13.9 & \textbf{4.8} & \textbf{8.0} & 32.0 & \textbf{37.5} \\
          + \SDRL{}-freq       & \textbf{65.5} & \textbf{71.0} & \textbf{49.9} & \textbf{58.0} & 9.7 & \textbf{15.5} & 4.1 & 4.9 & \textbf{32.3} & 37.4 \\
          \midrule
          \multicolumn{11}{c}{Qwen3-4B-Base} \\
          \midrule
          \emph{DAPO}      & 82.9 & 84.4 & 71.9 & 81.4 & 24.8 & 31.1 & 24.9 & 27.8 & 51.1 & 56.2 \\
          + \SDRL{}-rand      & \textbf{83.7} & 85.8 & 72.3 & \textbf{84.1} & 25.9 & \textbf{35.1} & 20.4 & 27.6 & 50.6 & 58.2 \\
          + \SDRL{}-freq       & 82.2 & \textbf{86.2} & \textbf{74.2} & 83.0 & \textbf{26.6} & 34.4 & \textbf{26.6} & \textbf{37.2} & \textbf{52.4} & \textbf{60.2} \\
          \midrule
          \multicolumn{11}{c}{Qwen3-8B-Base} \\
          \midrule
          \emph{DAPO}      & 83.2 & 85.4 & 77.7 & 86.8 & 26.1 & 33.0 & 23.6 & 26.6 & 52.7 & 58.0 \\
          + \SDRL{}-rand   & 83.4 & 86.6 & 78.3 & 90.5 & 30.1 & 38.4 & \textbf{27.3} & \textbf{35.8} & 54.8 & 62.8 \\
          + \SDRL{}-freq   & \textbf{85.4} & \textbf{88.4} & \textbf{80.6} & \textbf{93.6} & \textbf{32.6} & \textbf{43.4} & 26.6 & 34.4 & \textbf{56.3} & \textbf{65.0} \\
          \bottomrule
        \end{tabular}
        }
    \end{small}
  \end{center}
  \vskip -0.1in
\end{table*}

\subsection{Single-Agent Reasoning}
\label{single_exp}

We follow prior reasoning LLM evaluation protocols~\cite{yu2025dapo} and evaluate each model as a single agent. Table~\ref{tab:main-single} reports both $\mathrm{mean}@K$ and $\mathrm{maj}@K$ across benchmarks. Overall, \SDRL{} improves single-agent reasoning across all three model architectures. This shows that training with \SDRL{} not only prepares models for debate, but also strengthens standalone problem-solving ability.

For Qwen2.5-3B, \SDRL{}-rand increases the average $\mathrm{mean}@K$ from $30.8$ to $32.0$ and the average $\mathrm{maj}@K$ from $36.2$ to $37.5$. \SDRL{}-freq further improves the average $\mathrm{mean}@K$ to $32.3$ and reaches a comparable average $\mathrm{maj}@K$ of $37.4$. For Qwen3-4B-Base, \SDRL{}-freq gives the strongest overall results, improving the average $\mathrm{mean}@K$ from $51.1$ to $52.4$ and the average $\mathrm{maj}@K$ from $56.2$ to $60.2$.
The advantage of \SDRL{}-freq becomes more pronounced on Qwen3-8B-Base. It improves the average $\mathrm{mean}@K$ from $52.7$ to $56.3$ and the average $\mathrm{maj}@K$ from $58.0$ to $65.0$, achieving the best average result among all methods. It also obtains the best performance on MATH500, AMC23, and AIME24. In particular, AIME24 $\mathrm{maj}@32$ improves from $33.0$ to $43.4$, and AMC23 $\mathrm{maj}@32$ improves from $86.8$ to $93.6$.

Overall, the single-agent results demonstrate that \SDRL{} improves reasoning LLMs beyond the multi-agent setting. By jointly optimizing initial and debate-conditioned responses, \SDRL{} improves the quality of direct responses while preserving the capabilities needed for effective debate. Among the two pairing strategies, \SDRL{}-freq provides the strongest overall performance, suggesting that frequency-based debate construction also benefits standalone reasoning.

\subsection{Further Analysis}
\label{further_analysis}
As \SDRL{} with frequency-based pairing yields better overall performance than random pairing, we conduct additional experiments using \SDRL{}+freq. More experiments, including the measurement of private critique $\delta$, can be found in Appendix.

\textbf{Debate Rounds.} We evaluate more debate rounds in the decentralized MAD setting with $5$ agents, and report the results in Figure~\ref{fig:round}. Compared with the DAPO baseline, \SDRL{} achieves higher accuracy across all debate rounds. Moreover, \SDRL{} typically yields larger improvements as the number of rounds increases. Except for MATH500, the debate performance of both DAPO and \SDRL{} increases in the early rounds and then declines, which is consistent with our theoretical analysis that debate improves in
early rounds but peaks quickly and can decline later. More analysis on MATH500 can be found in the Appendix.

\begin{figure*}[!t]
  \begin{center}
    \centerline{\includegraphics[width=0.8\textwidth]{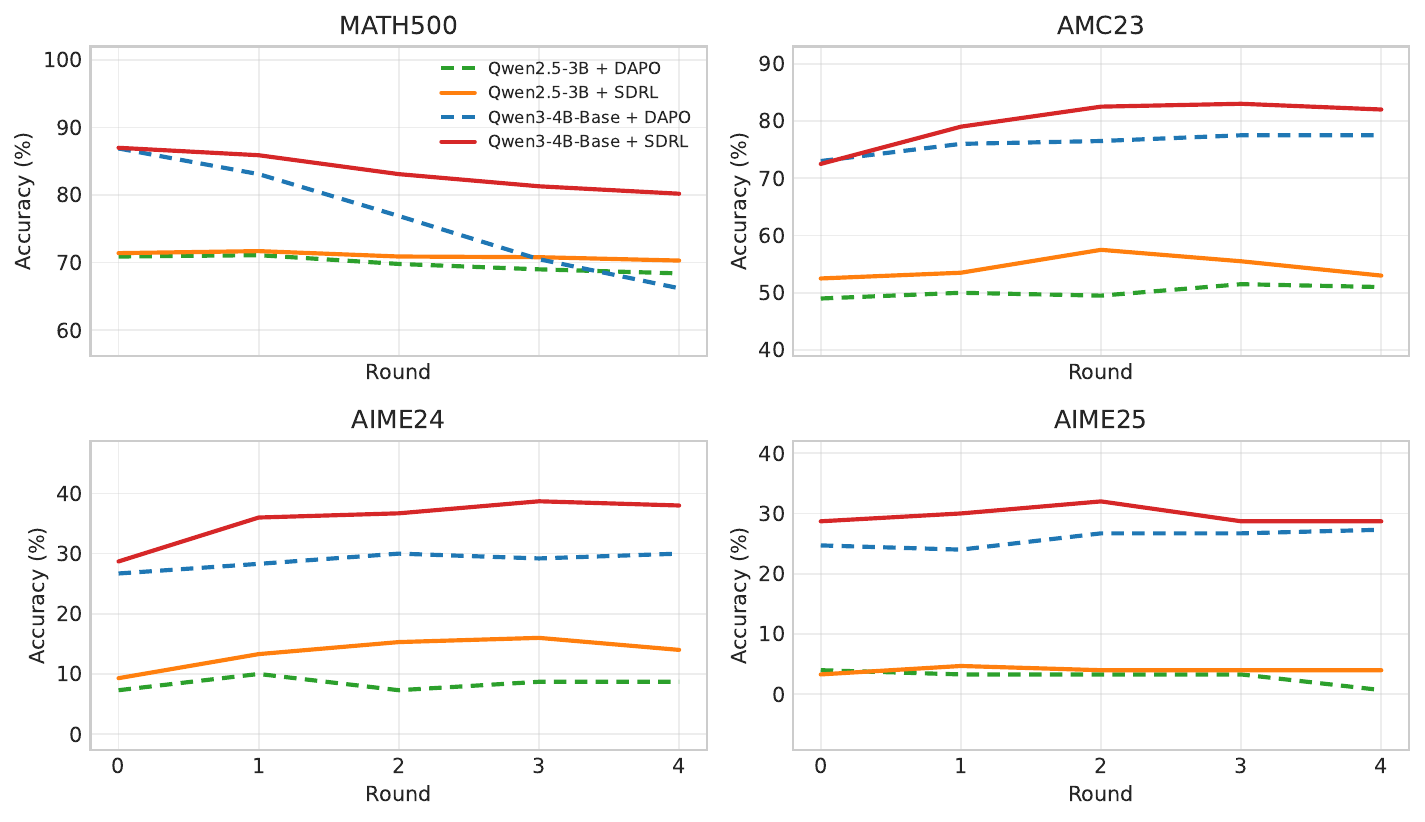}}
    \caption{
      Debate performance of $5$ agents in different debate rounds under the decentralized MAD setting.
    }
    \label{fig:round}
  \end{center}
  \vspace{-25pt}
\end{figure*}

\textbf{Number of Agents.} We report experimental results for the decentralized MAD system with varying numbers of agents under a single debate round, as shown in Table~\ref{tab:num_agents}. Overall, \SDRL{} consistently outperforms the DAPO baseline across all agent settings. Increasing the number of agents from $3$ to $5$ leads to consistent improvements in average debate performance. However, in the $7$-agent setting on AMC23, AIME24, and AIME25, performance drops because we reduce the max response tokens during generation to keep the debate sequence within the model context window, which results in truncated and incomplete responses.

\textbf{Additional Debate Frameworks.} Following~\cite{choi2025debate}, we evaluate \SDRL{} under additional debate frameworks. These include Sparse MAD~\cite{li2024improving}, a variant of decentralized MAD that uses a sparse communication topology to improve efficiency, and Centralized MAD~\cite{guo2024large}, where a central agent aggregates peers' responses and produces an updated response at each round. The results are reported in Table~\ref{tab:debate-framework}. Overall, \SDRL{} outperforms the DAPO baseline in both debate accuracy and the improvement brought by debate across all settings, demonstrating the robustness of \SDRL{}-trained models under different debate frameworks.

\begin{table*}[!t]
  \caption{Ablation on the number of agents $N$ in decentralized MAD setting. \textit{Maj} denotes the majority-vote accuracy of the agents' direct responses to the question. \textit{Debate} denotes the performance of the decentralized multi-agent system after debate round $1$. $\Delta$ is the difference between \textit{Maj} and \textit{Debate}.}
  \label{tab:num_agents}
  \begin{center}
    \resizebox{\textwidth}{!}{%
      \begin{tabular}{lccccccccccccccc}
        \toprule
        Method
          & \multicolumn{3}{c}{MATH500}
          & \multicolumn{3}{c}{AMC23}
          & \multicolumn{3}{c}{AIME24}
          & \multicolumn{3}{c}{AIME25}
          & \multicolumn{3}{c}{Avg.} \\
        \cmidrule(lr){2-4}\cmidrule(lr){5-7}\cmidrule(lr){8-10}\cmidrule(lr){11-13}\cmidrule(lr){14-16}
          & Maj & Debate & $\Delta$
          & Maj & Debate & $\Delta$
          & Maj & Debate & $\Delta$
          & Maj & Debate & $\Delta$
          & Maj & Debate & $\Delta$ \\
        \midrule
        \multicolumn{16}{c}{\textbf{Qwen2.5-3B}} \\
        \midrule
        DAPO (N=3) & 67.5 & 68.2 & 0.7  & 48.0 & 50.0 & 2.0 & 4.0 & 6.7 & 2.7 & 1.3 & 1.3 & 0.0 & 30.2 & 31.6 & 1.4 \\
        DAPO (N=5) & 70.9 & 71.1 & 0.2  & 49.0 & 50.0 & 1.0  & 7.3 & 10.0 & 2.7  & 4.0 & 3.3 & -0.7 & 32.8 & 33.6 & 0.8 \\
        DAPO (N=7) & 72.4 & 71.9 & -0.5 & 49.0 & 51.0 & 2.0 & 6.7 & 9.3 & 2.7 & 3.3 & 2.7 & -0.7 & 32.9 & 33.7 & 0.9 \\
        +\SDRL{} (N=3) & 66.2 & 67.3 & 1.1  & 51.0 & 55.0 & 4.0 & 11.3 & 13.3 & 2.0 & 1.3 & 4.7 & 3.3 & 32.5 & 35.1 & 2.6 \\
        +\SDRL{} (N=5) & 71.4 & 71.7 & 0.3  & 52.5 & 53.5 & 1.0  & 9.3 & 13.3 & 4.0  & 3.3 & 4.7 & 1.4  & 34.1 & 35.8 & 1.7 \\
        +\SDRL{} (N=7) & 71.5 & 71.9 & 0.4  & 53.0 & 55.0 & 2.0 & 12.0 & 14.7 & 2.7 & 4.7 & 6.0 & 1.3 & 35.3 & 36.9 & 1.6 \\
        \midrule
        \multicolumn{16}{c}{\textbf{Qwen3-4B-Base}} \\
        \midrule
        DAPO (N=3) & 85.2 & 82.8 & -2.4 & 68.0 & 70.5 & 2.5 & 20.7 & 25.3 & 4.7 & 24.7 & 23.3 & -1.3 & 49.6 & 50.5 & 0.9 \\
        DAPO (N=5)    & 86.9 & 83.1 & -3.8 & 73.0 & 76.0 & 3.0 & 26.7 & 28.3 & 1.6 & 24.7 & 24.0 & -0.7 & 52.8 & 52.9 & 0.1 \\
        DAPO (N=7) & 87.7 & 84.0 & -3.7 & 72.0 & 76.0 & 4.0 & 23.3 & 26.0 & 2.7 & 23.3 & 24.7 & 1.3 & 51.6 & 52.7 & 1.1 \\
        +\SDRL{} (N=3) & 84.4 & 83.6 & -0.7 & 70.0 & 76.7 & 6.7 & 24.7 & 34.0 & 9.3 & 20.0 & 23.3 & 3.3 & 49.8 & 54.4 & 4.7 \\
        +\SDRL{} (N=5)    & 87.0 & 85.9 & -1.1 & 72.5 & 79.0 & 6.5 & 28.7 & 36.0 & 7.3 & 28.7 & 30.0 & 1.3 & 54.2 & 57.7 & 3.5 \\
        +\SDRL{} (N=7) & 88.1 & 85.9 & -2.2 & 70.0 & 81.7 & 11.7 & 23.3 & 32.0 & 8.7 & 22.7 & 28.7 & 6.0 & 51.0 & 57.1 & 6.1 \\
        \bottomrule
      \end{tabular}%
    }
  \end{center}
  \vskip -0.1in
\end{table*}

\begin{table*}[t]
  \caption{Multi-agent debate performance under different debate settings. The debate system contains $5$ agents. \textit{Maj} denotes the majority-vote accuracy of the agents' direct responses to the question. \textit{Debate} denotes the performance of the decentralized multi-agent system after debate round $1$. $\Delta$ is the difference between \textit{Maj} and \textit{Debate}.}
  \label{tab:debate-framework}
  \begin{center}
    \resizebox{\textwidth}{!}{
      \begin{tabular}{lccccccccccccccc}
        \toprule
        Method
          & \multicolumn{3}{c}{MATH500}
          & \multicolumn{3}{c}{AMC23}
          & \multicolumn{3}{c}{AIME24}
          & \multicolumn{3}{c}{AIME25}
          & \multicolumn{3}{c}{Avg.} \\
        \cmidrule(lr){2-4}\cmidrule(lr){5-7}\cmidrule(lr){8-10}\cmidrule(lr){11-13}\cmidrule(lr){14-16}
          & Maj & Debate & $\Delta$
          & Maj & Debate & $\Delta$
          & Maj & Debate & $\Delta$
          & Maj & Debate & $\Delta$
          & Maj & Debate & $\Delta$ \\
        \midrule
        \multicolumn{16}{c}{\textbf{Qwen2.5-3B}} \\
        \midrule
        DAPO (sparse)              & 71.4 & 70.8 & -0.6 & 47.5 & 44.2 & -3.3 & 10.0 & 6.7 & -3.3 & 3.3 & 3.3 & 0.0 & 33.1 & 31.3 & -1.8 \\
        DAPO (centralized)          & 62.5 & 63.2 & 0.7  & 38.3 & 42.5 & 4.2  & 6.7  & 4.4 & -2.2 & 3.3 & 5.6 & 2.2  & 27.7 & 28.9 & 1.2 \\
        +\SDRL{} (sparse)      & 70.5 & 71.0 & 0.5  & 48.3 & 52.5 & 4.2  & 6.7  & 8.9 & 2.2  & 4.4 & 4.4 & 0.0  & 32.5 & 34.2 & 1.7 \\
        +\SDRL{} (centralized)  & 64.5 & 65.5 & 1.0  & 45.0 & 52.5 & 7.5  & 4.4  & 10.0 & 5.6 & 4.4 & 8.9 & 4.4  & 29.6 & 34.2 & 4.6 \\
        \midrule
        \multicolumn{16}{c}{\textbf{Qwen3-4B-Base}} \\
        \midrule
        DAPO (sparse)              & 86.5 & 86.0 & -0.6 & 70.5 & 78.0 & 7.5 & 22.7 & 24.7 & 2.0 & 24.0 & 25.3 & 1.3 & 50.9 & 53.5 & 2.6 \\
        DAPO (centralized)          & 80.7 & 79.5 & -1.2 & 64.0 & 72.0 & 8.0 & 23.3 & 27.3 & 4.0 & 21.4 & 22.0 & 0.6 & 47.4 & 50.2 & 2.9 \\
        +\SDRL{} (sparse)      & 86.7 & 86.6 & -0.1 & 69.5 & 78.5 & 9.0 & 28.7 & 36.0 & 7.3 & 28.7 & 33.3 & 4.6 & 53.4 & 58.6 & 5.2 \\
        +\SDRL{} (centralized)  & 80.5 & 81.0 & 0.5  & 66.5 & 78.0 & 11.5 & 25.3 & 34.0 & 8.7 & 25.3 & 26.7 & 1.4 & 49.4 & 54.9 & 5.5 \\
        \bottomrule
      \end{tabular}
    }
  \end{center}
  \vskip -0.1in
\end{table*}

\section{Conclusion}

We introduced a theoretical framework that disentangles debate-driven improvement from majority-voting effects in MAD. We then propose \emph{Self-Debate Reinforcement Learning} (\SDRL{}), which trains a \emph{single} reasoning model to be effective both as a standalone solver and as a debate participant by jointly optimizing initial and debate-conditioned responses with verifiable rewards. Experiments across multiple model architectures, reasoning benchmarks, and debate frameworks show that \SDRL{} consistently improves multi-agent debate performance while simultaneously strengthening single-agent accuracy.

\bibliographystyle{plainnat}
\bibliography{ref}

\appendix
\onecolumn
\begin{center}
    {\LARGE \bf Appendix}
\end{center}

\section{Related Work}
\label{related_work}

\paragraph{Reinforcement Learning with Verifiable Rewards (RLVR).}
Motivated by the success of DeepSeek-R1~\cite{guo2025deepseek}, a growing body of work adopts RLVR to strengthen the reasoning capabilities of LLMs across a range of domains, including math and STEM problem solving~\cite{wang2025beyond, liu2025prorl}, search~\cite{jin2025search, chen2025learning}, agentic tool use~\cite{qian2025toolrl,zhang2025nemotron}, and reward modeling~\cite{chen2025rm, wang2025llava,guo2025reward}. In parallel, several test-time scaling techniques have been proposed to further improve reasoning performance, such as self-consistency~\cite{wang2022self}, self-verification~\cite{zhang2025incentivizing,liu2025trust}, and parallel thinking~\cite{zheng2025parallel}. However, most existing methods focus on single-agent performance, where each reasoning trajectory is generated by a single policy model, and therefore do not explicitly prepare models for multi-agent debate.

\paragraph{Multi-Agent Debate (MAD).}
Multi-agent systems have proven effective for complex LLM applications that require sophisticated reasoning~\cite{talebirad2023multi,li2024survey,han2024llm,xiong2025multi}, and several studies propose different MAD settings~\cite{du2023improving,liang2024encouraging,liu2024groupdebate,smit2023should}. However, existing MAD methods do not consistently yield improvements~\cite{smit2023should,ma2025judging} and often fail to outperform strong majority-voting baselines~\cite{choi2025debate}. Moreover, even state-of-the-art closed-source LLMs can struggle to effectively incorporate conflicting opinions~\cite{kumaran2025overconfidence,li2025off}. Recent work attempts to improve MAD with reinforcement learning, but typically either trains an additional aggregator model~\cite{qi2025learning,zhao2025majority} or optimizes an entire multi-agent collaboration system, where each agent has a specialized role~\cite{liao2025marft,liu2025llm,marti2025,liu2025spiral}s. In contrast, our goal is to train a \emph{single} general-purpose LLM with debate capability—one that remains strong as an individual solver while being explicitly prepared to revise its reasoning when exposed to diverse opinions from other models or users.

\section{Experimental Details}
\subsection{Implementation Details}
We provide the full hyperparameter configuration of DAPO~\cite{yu2025dapo}, which is used both as our baseline and as the underlying optimizer for \SDRL{}. Across all experiments, we set the KL coefficient to $0$ and use asymmetric clipping with $\epsilon_{low}=0.2$ and $\epsilon_{high}=0.28$. We set the maximum response length to $8{,}196$ tokens. The overlong buffer is fixed at $2{,}048$ tokens with an overlong penalty factor of $1$. The training mini-batch size is $128$, corresponding to $16$ gradient updates per rollout step. We use $200$ prompt generation steps, corresponding to $3{,}200$ policy update steps for the DAPO baseline. \SDRL{} has more policy update steps due to the additional debate training. To improve training efficiency, at each prompt generation step we randomly sample $128$ prompts to construct debate pairs for \SDRL{}. 

\section{Evaluation Details}
For multi-agent debate evaluation, we use the codebase from~\citet{choi2025debate} to perform all MAD settings. As 8B models require longer reasoning, YaRN~\citep{peng2023yarn} is used to support a larger debate context window on AIME24 and AIME25 tasks. For single-agent evaluation, we use the \textsc{verl} framework~\cite{sheng2024hybridflow}. On MATH500, we use \textsc{Math-Verify} to extract final answers. For AMC and AIME, we follow the DAPO evaluation protocol for prompting and final-answer extraction. 

\section{Prompts for debate}
The prompt used for \SDRL{} training is shown in Figure~\ref{fig:train-prompt}. For MAD evaluation, we follow~\cite{choi2025debate} and use their math-task prompt template, shown in Figure~\ref{fig:evaluate-prompt}. For brevity, we illustrate the setup with three agents in the debate.

\begin{figure*}[h]
  \begin{center}
    \centerline{\includegraphics[width=\textwidth]{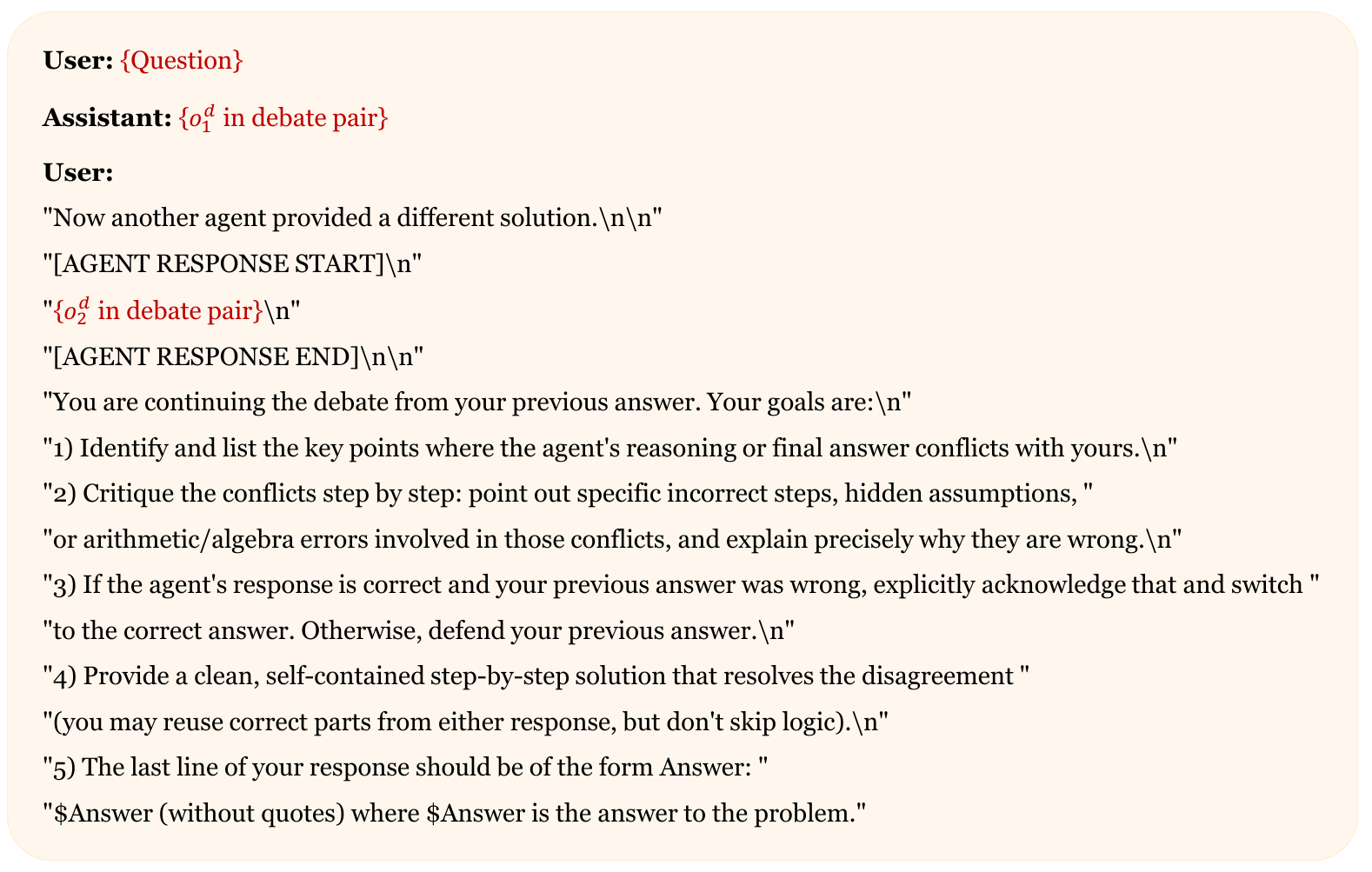}}
    \caption{
      Prompt for debate training.
    }
    \label{fig:train-prompt}
  \end{center}
  \vspace{-15pt}
\end{figure*}

\begin{figure*}[h]
  \begin{center}
    \centerline{\includegraphics[width=\textwidth]{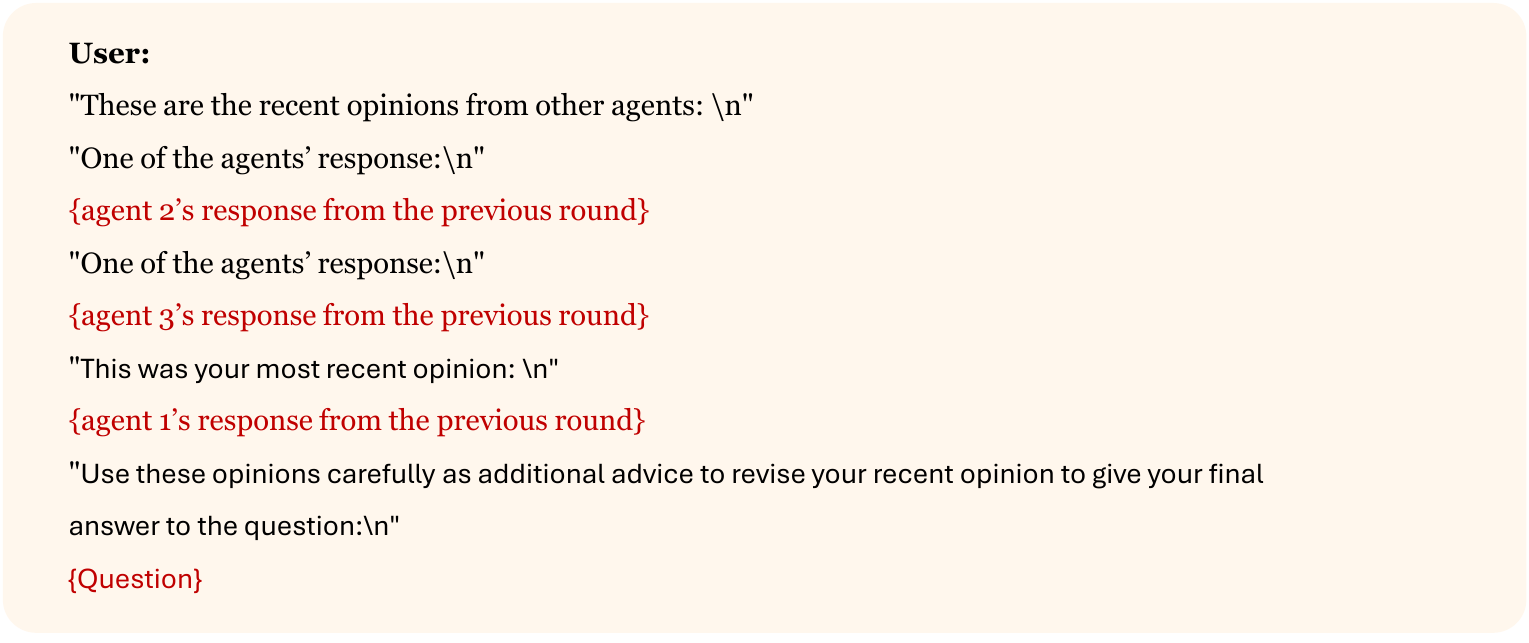}}
    \caption{
      Prompt for Multi-Agent evaluation. Three agents debate is shown for brevity.
    }
    \label{fig:evaluate-prompt}
  \end{center}
  \vspace{-15pt}
\end{figure*}

\section{Additional Experiments}

\subsection{Training Dynamics of \SDRL{} }

To better understand \SDRL{} training dynamics, we visualize (i) the number of constructed debate prompts and (ii) the number of debate-conditioned responses used at each prompt-generation step on the training set. The results are shown in Figure~\ref{fig:visual-training}. In the figure, \textit{Filter} refers to our filtering procedure that removes debate responses with zero advantage. \textit{Debate Prompts} denotes the number of debate prompts retained for policy updates, and \textit{Debate Samples} denotes the number of debate-conditioned responses used for policy updates. The results indicate that most debate training occurs early in training. As training progresses, debating on the training set becomes easier, and an increasing fraction of debate responses receives zero advantage, providing no learning signal. Since we do not apply oversampling for debate responses, the additional computational cost of \SDRL{} remains modest.

In the bottom row, \textit{Debate Accuracy} denotes the accuracy of second-turn responses conditioned on debate contexts, and \textit{Debate Accuracy Improvement} measures the accuracy gain of second-turn responses over the initial responses. These trends suggest that \SDRL{} effectively equips the model with debate capability, enabling it to revise answers productively when exposed to alternative reasoning trajectories.

\begin{figure*}[h]
  \begin{center}
    \centerline{\includegraphics[width=\textwidth]{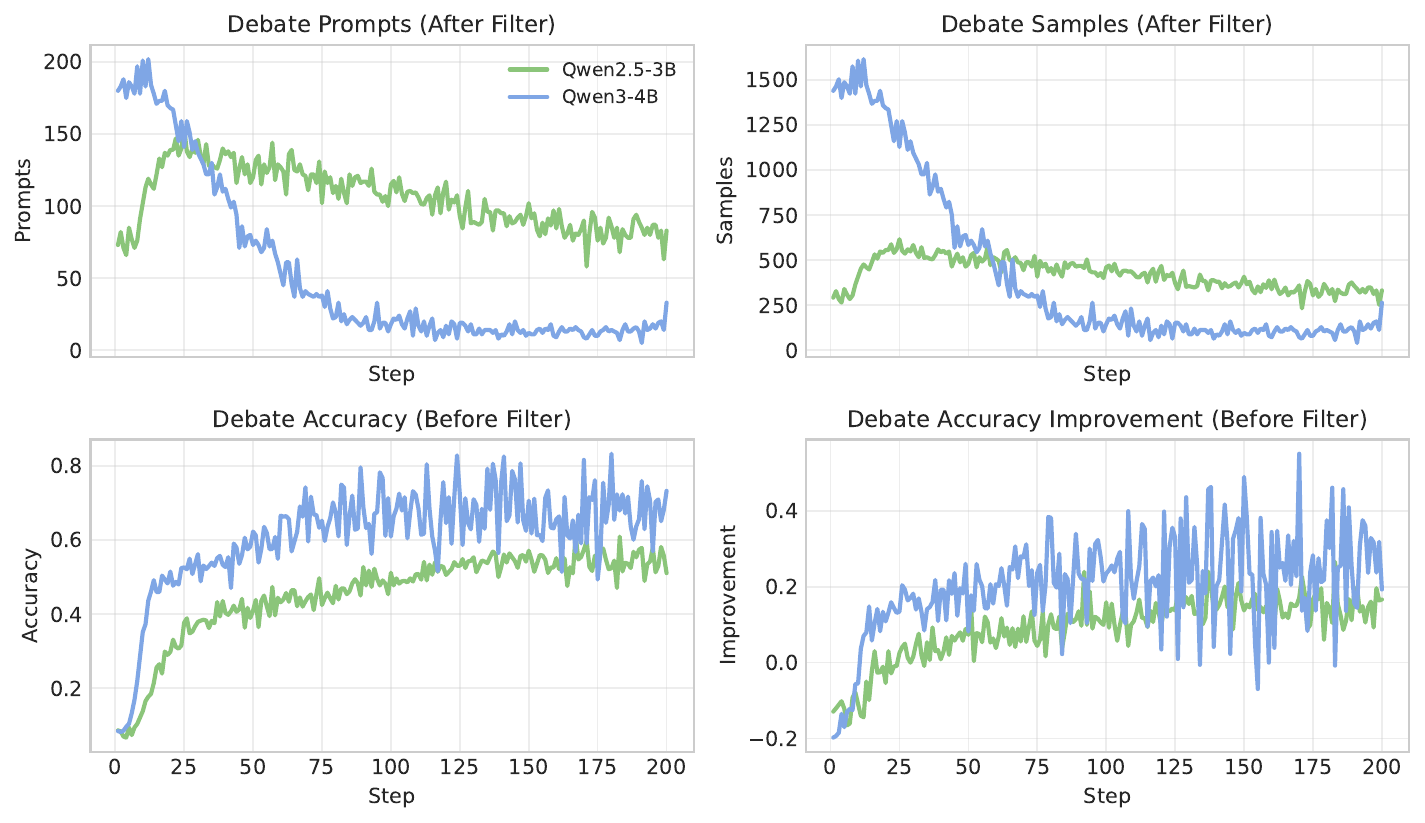}}
    \caption{
      Training dynamics of debate-related values using \SDRL{}.
    }
    \label{fig:visual-training}
  \end{center}
  \vspace{-15pt}
\end{figure*}

\subsection{Reduced Training Budget}
As \SDRL{} requires additional computation to train debate pairs, we also evaluate \SDRL{} under a reduced training budget on Qwen3-4B-Base. Table~\ref{tab:reduced_budget} reports results for \SDRL{}+freq trained with only $125$ prompt-generation steps, which uses fewer training resources than the DAPO baseline trained with $200$ steps. Despite the reduced budget, \SDRL{}+freq still outperforms the well-trained DAPO baseline in debate performance, while remaining worse than the fully trained \SDRL{}, indicating that \SDRL{} has not saturated at $125$ steps. These results further support that the gains of \SDRL{} come from training on debate data, rather than simply from increased overall training compute.

\begin{table*}[!h]
  \caption{Reduced training budgets for \SDRL{}+freq with different numbers of training steps, evaluated in the decentralized multi-agent debate setting. The debate system contains $5$ agents. \textit{Maj} denotes the majority-vote accuracy of the agents' direct responses to the question. \textit{Debate} denotes the performance of the decentralized multi-agent system after debate round $1$. $\Delta$ is the difference between \textit{Maj} and \textit{Debate}. The best results in each column are bolded.}
  \label{tab:reduced_budget}
  \begin{center}
    \resizebox{\textwidth}{!}{
      \begin{tabular}{lccccccccccccccc}
        \toprule
        Method
          & \multicolumn{3}{c}{MATH500}
          & \multicolumn{3}{c}{AMC23}
          & \multicolumn{3}{c}{AIME24}
          & \multicolumn{3}{c}{AIME25}
          & \multicolumn{3}{c}{Avg.} \\
        \cmidrule(lr){2-4}\cmidrule(lr){5-7}\cmidrule(lr){8-10}\cmidrule(lr){11-13}\cmidrule(lr){14-16}
          & Maj & Debate & $\Delta$
          & Maj & Debate & $\Delta$
          & Maj & Debate & $\Delta$
          & Maj & Debate & $\Delta$
          & Maj & Debate & $\Delta$ \\
        \midrule
        \multicolumn{16}{c}{Qwen3-4B-Base} \\
        \midrule
        DAPO-$200$ steps
          & 86.9 & 83.1 & -3.8
          & \textbf{73.0} & 76.0 & 3.0
          & 26.7 & 28.3 & 1.6
          & 24.7 & 24.0 & -0.7
          & 52.8 & 52.9 & 0.1 \\
        + \SDRL{}-$125$ steps
          & 86.6 & \textbf{86.0} & \textbf{-0.6}
          & 68.5 & 78.0 & \textbf{9.5}
          & 24.7 & 32.0 & \textbf{7.3}
          & 27.3 & 27.3 & 0.0
          & 51.8 & 55.8 & \textbf{4.0} \\
        + \SDRL{}-$200$ steps
          & \textbf{87.0} & 85.9 & -1.1
          & 72.5 & \textbf{79.0} & 6.5
          & \textbf{28.7} & \textbf{36.0} & \textbf{7.3}
          & \textbf{28.7} & \textbf{30.0} & \textbf{1.3}
          & \textbf{54.2} & \textbf{57.7} & 3.5 \\
        \bottomrule
      \end{tabular}
    }
  \end{center}
  \vskip -0.1in
\end{table*}

\subsection{Broader Reasoning Tasks}
\label{app:broader_reasoning}

Our main experiments focus on math reasoning, where answer correctness can be reliably verified. To examine whether the learned debate behavior transfers beyond math, we further evaluate the math-trained models on broader reasoning tasks. We consider CommonSenseQA~\citep{talmor2019commonsenseqa}, HellaSwag~\citep{zellers2019hellaswag}, HH-RLHF~\cite{bai2022training}, MMLU Professional Medicine~\cite{hendrycks2020measuring}, and the open-ended summarization benchmark CNN/DailyMail~\cite{see2017get}. We directly evaluate the math-trained models on these benchmarks.

Table~\ref{tab:app-broader-reasoning} reports results on multiple-choice and preference-style reasoning tasks. Overall, \SDRL{} improves out-of-domain reasoning performance over DAPO. \SDRL{}-freq achieves the best average \textit{Maj} accuracy of $73.6$ and the best average \textit{Debate} accuracy of $73.8$, improving DAPO by $1.7$ and $3.1$ points, respectively. \SDRL{}-rand also improves the average \textit{Debate} accuracy from $70.7$ to $73.6$. The gains are especially clear on CommonSenseQA and HH-RLHF, suggesting that the debate capability learned from math training can transfer to broader reasoning settings.

Table~\ref{tab:app-cnndm} reports ROUGE scores on CNN/DailyMail. Since summarization is open-ended and differs substantially from verifiable math reasoning, the gains are more modest. Still, \SDRL{}-freq achieves the best round-$0$ ROUGE scores and improves round-$1$ ROUGE-2 and ROUGE-L over DAPO. These results suggest that \SDRL{} can provide useful generalization beyond math, while open-ended generation may require more domain-aligned debate training.

\begin{table*}[t]
  \caption{Broader reasoning evaluation of math-trained models. We report the majority-vote accuracy of direct responses (\textit{Maj}), the accuracy after one debate round (\textit{Debate}), and their difference $\Delta$. Best results are bolded within each dataset.}
  \label{tab:app-broader-reasoning}
  \begin{center}
    \resizebox{\textwidth}{!}{
      \begin{tabular}{lccccccccccccccc}
        \toprule
        Method
          & \multicolumn{3}{c}{CommonSenseQA}
          & \multicolumn{3}{c}{HellaSwag}
          & \multicolumn{3}{c}{HH-RLHF}
          & \multicolumn{3}{c}{MMLU Pro. Med.}
          & \multicolumn{3}{c}{Avg.} \\
        \cmidrule(lr){2-4}\cmidrule(lr){5-7}\cmidrule(lr){8-10}\cmidrule(lr){11-13}\cmidrule(lr){14-16}
          & Maj & Debate & $\Delta$
          & Maj & Debate & $\Delta$
          & Maj & Debate & $\Delta$
          & Maj & Debate & $\Delta$
          & Maj & Debate & $\Delta$ \\
        \midrule
        \emph{DAPO}
          & 82.2 & 82.4 & 0.2
          & 79.8 & 77.4 & -2.4
          & 43.7 & 40.9 & -2.8
          & 81.9 & 82.2 & 0.3
          & 71.9 & 70.7 & -1.2 \\
        + \SDRL{}-rand
          & \textbf{83.3} & \textbf{85.3} & \textbf{2.0}
          & 79.7 & 78.8 & -0.9
          & 46.0 & 46.1 & \textbf{0.1}
          & \textbf{83.2} & \textbf{84.0} & 0.8
          & 73.1 & 73.6 & \textbf{0.5} \\
        + \SDRL{}-freq
          & 82.9 & 83.3 & 0.4
          & \textbf{80.6} & \textbf{79.8} & \textbf{-0.8}
          & \textbf{49.5} & \textbf{48.9} & -0.6
          & 81.4 & 83.0 & \textbf{1.6}
          & \textbf{73.6} & \textbf{73.8} & 0.2 \\
        \bottomrule
      \end{tabular}
    }
  \end{center}
  \vskip -0.1in
\end{table*}

\begin{table*}[t]
  \caption{Open-ended summarization results on CNN/DailyMail. Round $0$ denotes direct generation before debate, while rounds $1$ and $2$ denote outputs after debate. Best results are bolded within each round and metric.}
  \label{tab:app-cnndm}
  \begin{center}
    \resizebox{\textwidth}{!}{
      \begin{tabular}{lccccccccc}
        \toprule
        Method
          & \multicolumn{3}{c}{Round 0}
          & \multicolumn{3}{c}{Round 1}
          & \multicolumn{3}{c}{Round 2} \\
        \cmidrule(lr){2-4}\cmidrule(lr){5-7}\cmidrule(lr){8-10}
          & ROUGE-1 & ROUGE-2 & ROUGE-L
          & ROUGE-1 & ROUGE-2 & ROUGE-L
          & ROUGE-1 & ROUGE-2 & ROUGE-L \\
        \midrule
        \emph{DAPO}
          & 35.8 & 13.1 & 23.8
          & 34.5 & 11.9 & 22.4
          & 32.8 & 11.1 & 21.9 \\
        + \SDRL{}-rand
          & 35.6 & 12.7 & 23.7
          & \textbf{34.8} & 11.6 & 22.8
          & \textbf{33.7} & 11.2 & \textbf{22.5} \\
        + \SDRL{}-freq
          & \textbf{36.0} & \textbf{13.3} & \textbf{24.2}
          & 34.7 & \textbf{12.3} & \textbf{23.0}
          & 32.9 & \textbf{11.5} & 21.8 \\
        \bottomrule
      \end{tabular}
    }
  \end{center}
  \vskip -0.1in
\end{table*}

\subsection{Inference Time Debate Baselines}
\label{app:prompt_baselines}

We further compare \SDRL{} with representative inference time MAD baselines. These methods modify only the inference prompt and are orthogonal to our training recipe. We evaluate three variants on Qwen3-4B-Base: reflection prompts~\cite{huang2023large}, verification prompts~\cite{ling2023deductive}, and role-assigned debate prompts~\cite{liang2024encouraging}. Each prompt-based method is applied to both the DAPO baseline and the \SDRL{}-freq model.

Table~\ref{tab:app-prompt-reflect-verify} shows that prompt engineering alone does not account for the gains of \SDRL{}. Under reflection prompts, \SDRL{}-freq improves the average \textit{Debate} accuracy from $44.0$ to $51.4$ and changes the average debate gain from $-2.6$ to $1.5$. Under verification prompts, \SDRL{}-freq improves the average \textit{Debate} accuracy from $43.6$ to $49.1$ and increases the average debate gain from $-0.7$ to $3.7$. In both settings, \SDRL{}-freq achieves higher post-debate accuracy than DAPO on all four benchmarks.

Table~\ref{tab:app-prompt-role} reports the role-assigned debate results. \SDRL{}-freq improves the average accuracy across all debate rounds, increasing the round-$3$ average from $48.3$ to $50.7$. It also improves performance on MATH500, AMC23, and AIME24 across all three rounds. These results suggest that \SDRL{} learns debate capabilities that are complementary to prompt-based MAD strategies, and its gains cannot be fully replicated by prompt engineering alone.

\begin{table*}[t]
  \caption{Prompt-based reflection and verification baselines on Qwen3-4B-Base. We apply each prompt-based MAD strategy to both DAPO and \SDRL{}-freq. \textit{Maj} denotes the majority-vote accuracy before debate, \textit{Debate} denotes the accuracy after one debate round, and $\Delta$ denotes their difference. Best results are bolded within each prompt block.}
  \label{tab:app-prompt-reflect-verify}
  \begin{center}
    \resizebox{\textwidth}{!}{
      \begin{tabular}{lccccccccccccccc}
        \toprule
        Method
          & \multicolumn{3}{c}{MATH500}
          & \multicolumn{3}{c}{AMC23}
          & \multicolumn{3}{c}{AIME24}
          & \multicolumn{3}{c}{AIME25}
          & \multicolumn{3}{c}{Avg.} \\
        \cmidrule(lr){2-4}\cmidrule(lr){5-7}\cmidrule(lr){8-10}\cmidrule(lr){11-13}\cmidrule(lr){14-16}
          & Maj & Debate & $\Delta$
          & Maj & Debate & $\Delta$
          & Maj & Debate & $\Delta$
          & Maj & Debate & $\Delta$
          & Maj & Debate & $\Delta$ \\
        \midrule
        \multicolumn{16}{c}{Prompt-based reflection} \\
        \midrule
        \emph{DAPO}
          & \textbf{81.8} & 68.3 & -13.5
          & 64.0 & 67.0 & 3.0
          & 20.7 & 20.7 & \textbf{0.0}
          & 20.0 & 20.0 & 0.0
          & 46.6 & 44.0 & -2.6 \\
        + \SDRL{}-freq
          & 79.6 & \textbf{72.2} & \textbf{-7.4}
          & \textbf{68.0} & \textbf{80.0} & \textbf{12.0}
          & \textbf{26.0} & \textbf{26.0} & \textbf{0.0}
          & \textbf{26.0} & \textbf{27.3} & \textbf{1.3}
          & \textbf{49.9} & \textbf{51.4} & \textbf{1.5} \\
        \midrule
        \multicolumn{16}{c}{Prompt-based verification} \\
        \midrule
        \emph{DAPO}
          & 73.1 & 72.6 & -0.5
          & 64.5 & 62.5 & -2.0
          & \textbf{18.7} & 18.7 & 0.0
          & 20.7 & 20.7 & 0.0
          & 44.3 & 43.6 & -0.7 \\
        + \SDRL{}-freq
          & \textbf{78.3} & \textbf{82.6} & \textbf{4.3}
          & \textbf{65.5} & \textbf{70.5} & \textbf{5.0}
          & 16.7 & \textbf{21.3} & \textbf{4.6}
          & \textbf{21.3} & \textbf{22.0} & \textbf{0.7}
          & \textbf{45.5} & \textbf{49.1} & \textbf{3.7} \\
        \bottomrule
      \end{tabular}
    }
  \end{center}
  \vskip -0.1in
\end{table*}

\begin{table*}[t]
  \caption{Prompt-based role-assigned debate results on Qwen3-4B-Base. We report accuracy after each debate round. Best results are bolded within each benchmark and round.}
  \label{tab:app-prompt-role}
  \begin{center}
    \resizebox{\textwidth}{!}{
      \begin{tabular}{lccccccccccccccc}
        \toprule
        Method
          & \multicolumn{3}{c}{MATH500}
          & \multicolumn{3}{c}{AMC23}
          & \multicolumn{3}{c}{AIME24}
          & \multicolumn{3}{c}{AIME25}
          & \multicolumn{3}{c}{Avg.} \\
        \cmidrule(lr){2-4}\cmidrule(lr){5-7}\cmidrule(lr){8-10}\cmidrule(lr){11-13}\cmidrule(lr){14-16}
          & R1 & R2 & R3
          & R1 & R2 & R3
          & R1 & R2 & R3
          & R1 & R2 & R3
          & R1 & R2 & R3 \\
        \midrule
        \emph{DAPO}
          & 80.9 & 79.7 & 79.5
          & 65.8 & 66.7 & 66.7
          & 23.3 & 23.3 & 24.4
          & \textbf{22.4} & \textbf{22.4} & \textbf{22.4}
          & 48.1 & 48.0 & 48.3 \\
        + \SDRL{}-freq
          & \textbf{84.2} & \textbf{84.5} & \textbf{84.4}
          & \textbf{67.5} & \textbf{68.8} & \textbf{69.4}
          & \textbf{26.7} & \textbf{25.6} & \textbf{26.7}
          & 21.1 & 22.2 & 22.2
          & \textbf{49.9} & \textbf{50.3} & \textbf{50.7} \\
        \bottomrule
      \end{tabular}
    }
  \end{center}
  \vskip -0.1in
\end{table*}

\subsection{Private Critique Measurement}
\label{app:critique_quality}

The main results measure whether debate improves final accuracy. We further diagnose whether \SDRL{} improves the private critique $\delta$ and the quality of intermediate critique and revision. All results in this subsection use Qwen3-4B-Base.

Although it is hard to measure the private critique $\delta$ directly, we evaluate critique ability with LLM as a judge. Given a debate context, each model is asked to identify the main conflict, critique the competing reasoning step by step, decide whether to keep or revise its answer, and provide a revised solution. The judge scores each output using four criteria: 1. conflict identification(0-2): Does it correctly identify the actual points of disagreement? 2. critique correctness(0-3): Are the critiques mathematically and logically correct? 3. critique specificity(0-2): Does it point to concrete steps, assumptions, or errors instead of giving vague comments? 4. peer engagement and revision(0-2): Does it meaningfully engage with peer responses and make an appropriate keep or revise decision? Table~\ref{tab:app-critique-quality} shows that \SDRL{}-freq improves all four dimensions on average. The total critique score increases from $3.171$ to $3.738$, with consistent gains across all benchmarks. This provides direct evidence that \SDRL{} improves the quality of debate-state critique, rather than only improving final-answer accuracy.

We also report agent-level transition statistics after the first debate round. Table~\ref{tab:app-transition} measures how often an agent changes from a correct answer to an incorrect answer and from an incorrect answer to a correct answer. \SDRL{}-freq increases Incorrect$\to$Correct transitions on all four benchmarks and reduces Correct$\to$Incorrect transitions on three benchmarks. As a result, the useful revision frequency improves from $2.0$ to $7.2$ on average. These results suggest that \SDRL{} promotes productive revision after observing peer context, rather than merely increasing unstable answer changes.

\begin{table*}[t]
  \caption{Private critique quality judged on Qwen3-4B-Base. The judge evaluates conflict identification, critique correctness, critique specificity, and peer engagement. Best results are bolded within each benchmark.}
  \label{tab:app-critique-quality}
  \begin{center}
    \resizebox{\textwidth}{!}{
      \begin{tabular}{llccccc}
        \toprule
        Benchmark & Method
          & Conflict Identification
          & Critique Correctness
          & Critique Specificity
          & Peer Engagement
          & Total Score \\
        \midrule
        MATH500
          & \emph{DAPO}        & 0.932 & 1.366 & 0.789 & 0.852 & 3.939 \\
          & + \SDRL{}-freq     & \textbf{0.945} & \textbf{1.482} & \textbf{0.929} & \textbf{0.900} & \textbf{4.256} \\
        \midrule
        AMC23
          & \emph{DAPO}        & 1.000 & 0.970 & 0.625 & 0.670 & 3.265 \\
          & + \SDRL{}-freq     & \textbf{1.125} & \textbf{1.115} & \textbf{0.935} & \textbf{0.885} & \textbf{4.060} \\
        \midrule
        AIME24
          & \emph{DAPO}        & 1.087 & 0.460 & 0.653 & 0.640 & 2.840 \\
          & + \SDRL{}-freq     & \textbf{1.207} & \textbf{0.473} & \textbf{0.887} & \textbf{0.820} & \textbf{3.387} \\
        \midrule
        AIME25
          & \emph{DAPO}        & 1.073 & 0.313 & 0.680 & 0.573 & 2.639 \\
          & + \SDRL{}-freq     & \textbf{1.180} & \textbf{0.440} & \textbf{0.867} & \textbf{0.760} & \textbf{3.247} \\
        \midrule
        Avg.
          & \emph{DAPO}        & 1.023 & 0.777 & 0.687 & 0.684 & 3.171 \\
          & + \SDRL{}-freq     & \textbf{1.114} & \textbf{0.878} & \textbf{0.905} & \textbf{0.841} & \textbf{3.738} \\
        \bottomrule
      \end{tabular}
    }
  \end{center}
  \vskip -0.1in
\end{table*}

\begin{table*}[t]
  \caption{Agent-level transition statistics after one debate round on Qwen3-4B-Base. Correct$\to$Incorrect measures harmful revisions, while Incorrect$\to$Correct measures useful revisions. $\delta$ is defined as Incorrect$\to$Correct minus Correct$\to$Incorrect. Best results are bolded within each benchmark.}
  \label{tab:app-transition}
  \begin{center}
    \resizebox{0.82\textwidth}{!}{
      \begin{tabular}{llccc}
        \toprule
        Benchmark & Method
          & Correct$\to$Incorrect
          & Incorrect$\to$Correct
          & $\delta$ \\
        \midrule
        MATH500
          & \emph{DAPO}        & 7.5 & 5.8 & -1.6 \\
          & + \SDRL{}-freq     & \textbf{4.9} & \textbf{6.1} & \textbf{1.2} \\
        \midrule
        AMC23
          & \emph{DAPO}        & 3.5 & 8.5 & 5.0 \\
          & + \SDRL{}-freq     & \textbf{1.5} & \textbf{14.5} & \textbf{13.0} \\
        \midrule
        AIME24
          & \emph{DAPO}        & 1.3 & 4.7 & 3.3 \\
          & + \SDRL{}-freq     & \textbf{0.7} & \textbf{10.7} & \textbf{10.0} \\
        \midrule
        AIME25
          & \emph{DAPO}        & \textbf{0.7} & 2.0 & 1.3 \\
          & + \SDRL{}-freq     & 1.5 & \textbf{6.2} & \textbf{4.7} \\
        \midrule
        Avg.
          & \emph{DAPO}        & 3.3 & 5.3 & 2.0 \\
          & + \SDRL{}-freq     & \textbf{2.2} & \textbf{9.4} & \textbf{7.2} \\
        \bottomrule
      \end{tabular}
    }
  \end{center}
  \vskip -0.1in
\end{table*}

\subsection{Heterogeneous Multi-Agent Debate}
\label{app:heterogeneous_mad}

The main experiments use homogeneous multi-agent systems, where all agents share the same backbone and training variant. We further evaluate heterogeneous MAD systems composed of different base models and training variants. We use \texttt{Q3} to denote Qwen3-4B-Base and \texttt{Q2.5} to denote Qwen2.5-3B. We use \texttt{D}, \texttt{R}, and \texttt{F} to denote DAPO, \SDRL{}-rand, and \SDRL{}-freq, respectively.

Table~\ref{tab:app-heterogeneous} shows that \SDRL{}-trained agents remain effective in heterogeneous systems. The fully non-\SDRL{} heterogeneous baseline, \texttt{Q3-D}$\times3$+\texttt{Q2.5-D}$\times2$, reaches an average \textit{Debate} accuracy of $52.4$. In contrast, all mixed systems containing \SDRL{} agents achieve higher average \textit{Debate} accuracy. The best average \textit{Debate} accuracy is obtained by \texttt{Q3-F}$\times3$+\texttt{Q3-D}$\times2$, reaching $56.9$. The cross-architecture \SDRL{}-freq system, \texttt{Q3-F}$\times3$+\texttt{Q2.5-F}$\times2$, achieves the largest average debate gain of $\Delta=5.0$ and improves performance after debate on all four benchmarks. These results suggest that \SDRL{} improves debate behavior even when agents interact with heterogeneous peers.

\begin{table*}[t]
  \caption{Heterogeneous multi-agent debate results. Each system contains $5$ agents with different model architectures or training variants. \textit{Maj} denotes majority-vote accuracy before debate, \textit{Debate} denotes accuracy after one debate round, and $\Delta$ denotes their difference. Best \textit{Debate} results are bolded within each benchmark.}
  \label{tab:app-heterogeneous}
  \begin{center}
    \resizebox{\textwidth}{!}{
      \begin{tabular}{lccccccccccccccc}
        \toprule
        Composition
          & \multicolumn{3}{c}{MATH500}
          & \multicolumn{3}{c}{AMC23}
          & \multicolumn{3}{c}{AIME24}
          & \multicolumn{3}{c}{AIME25}
          & \multicolumn{3}{c}{Avg.} \\
        \cmidrule(lr){2-4}\cmidrule(lr){5-7}\cmidrule(lr){8-10}\cmidrule(lr){11-13}\cmidrule(lr){14-16}
          & Maj & Debate & $\Delta$
          & Maj & Debate & $\Delta$
          & Maj & Debate & $\Delta$
          & Maj & Debate & $\Delta$
          & Maj & Debate & $\Delta$ \\
        \midrule
        \texttt{Q3-F}$\times3$ + \texttt{Q2.5-F}$\times2$
          & 84.9 & 85.3 & 0.4
          & 73.3 & 78.3 & 5.0
          & 20.0 & 27.8 & 7.8
          & 23.3 & \textbf{30.0} & 6.7
          & 50.4 & 55.4 & \textbf{5.0} \\
        \texttt{Q3-F}$\times3$ + \texttt{Q3-D}$\times2$
          & 87.3 & \textbf{86.0} & -1.3
          & 70.8 & \textbf{80.8} & 10.0
          & 31.1 & \textbf{32.8} & 1.7
          & 26.7 & 27.8 & 1.1
          & 54.0 & \textbf{56.9} & 2.9 \\
        \texttt{Q3-D}$\times3$ + \texttt{Q2.5-D}$\times2$
          & 86.0 & 83.7 & -2.3
          & 65.0 & 78.3 & 13.3
          & 23.3 & 24.4 & 1.1
          & 25.3 & 23.3 & -2.0
          & 49.9 & 52.4 & 2.5 \\
        \texttt{Q3-F}$\times2$ + \texttt{Q3-R}$\times1$ + \texttt{Q2.5-F}$\times1$ + \texttt{Q2.5-R}$\times1$
          & 85.2 & 85.5 & 0.3
          & 67.5 & 73.3 & 5.8
          & 22.2 & 31.1 & 8.9
          & 24.4 & 28.9 & 4.5
          & 49.8 & 54.7 & 4.9 \\
        \bottomrule
      \end{tabular}
    }
  \end{center}
  \vskip -0.1in
\end{table*}

\subsection{Case Study}

We present a case study illustrating performance degradation on MATH500 over multiple debate rounds using the Qwen3-4B-Base model trained with \SDRL{}. For this example, the correct answer is $2k+2$. Figure~\ref{fig:case1}, Figure~\ref{fig:case3}, and Figure~\ref{fig:case2} show the detailed responses of the five agents under the decentralized MAD setting for the initial responses, debate round~1, and debate round~2, respectively. As shown, responses that yield the correct answer are substantially shorter than those that produce an incorrect final answer. After debate, more agents shift toward the incorrect answer accompanied by longer reasoning traces.

We hypothesize that this failure mode arises because MATH500 is relatively easy for Qwen3-4B, and the model often retrieves the correct answer primarily from memorized knowledge acquired during pre-training. In this case, long but incorrect reasoning traces can dominate the shared debate context, biasing subsequent generations toward the same incorrect trajectory. To support this observation, Table~\ref{tab:case_length} reports the average response lengths for correct and incorrect answers across debate rounds. The results show that incorrect responses are consistently longer than correct responses, which correlates with the observed accuracy drop as debate proceeds.

\begin{figure*}[h]
  \begin{center}
    \centerline{\includegraphics[width=\textwidth]{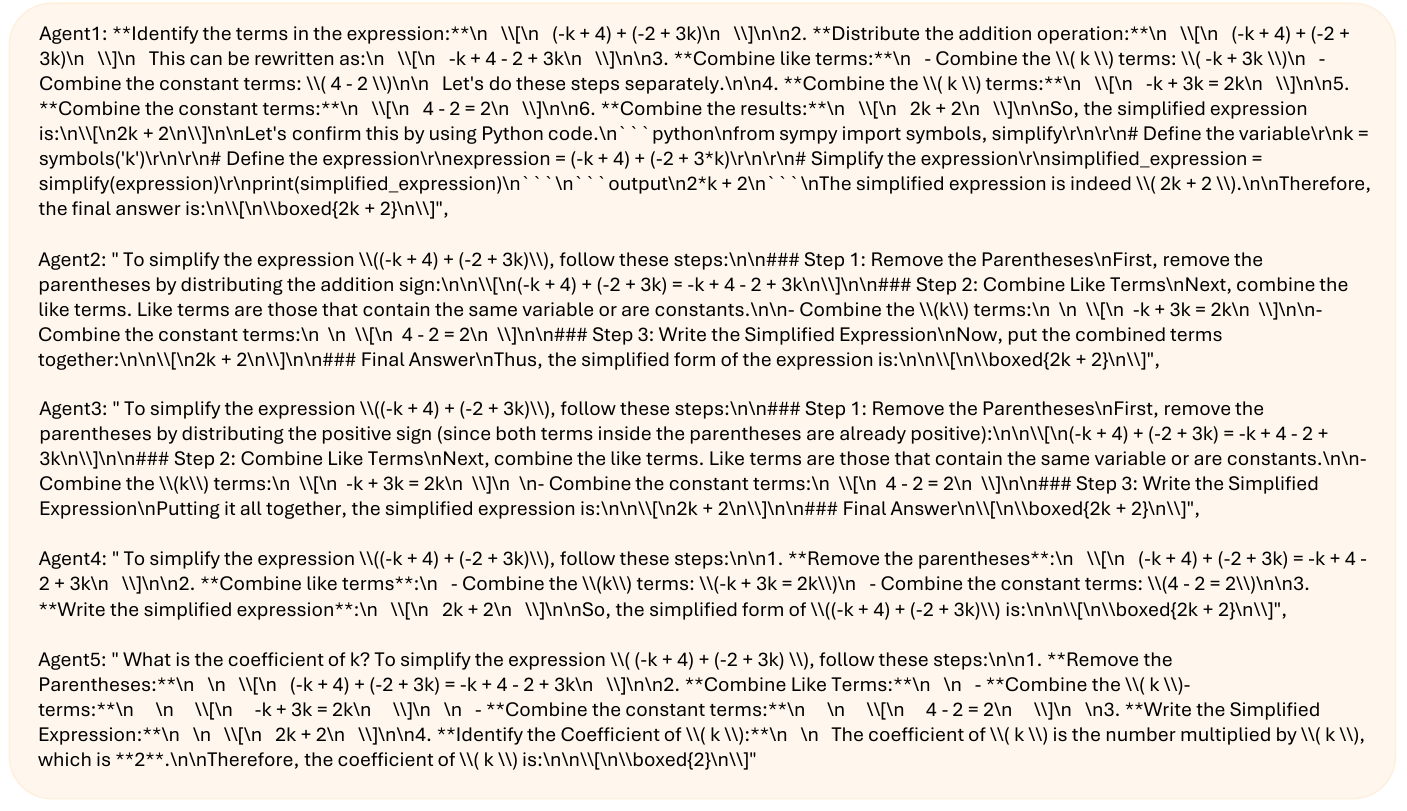}}
    \caption{
      Initial responses on MATH500 with $5$ agents.
    }
    \label{fig:case1}
  \end{center}
  \vspace{-15pt}
\end{figure*}

\begin{figure*}[h]
  \begin{center}
    \centerline{\includegraphics[width=\textwidth]{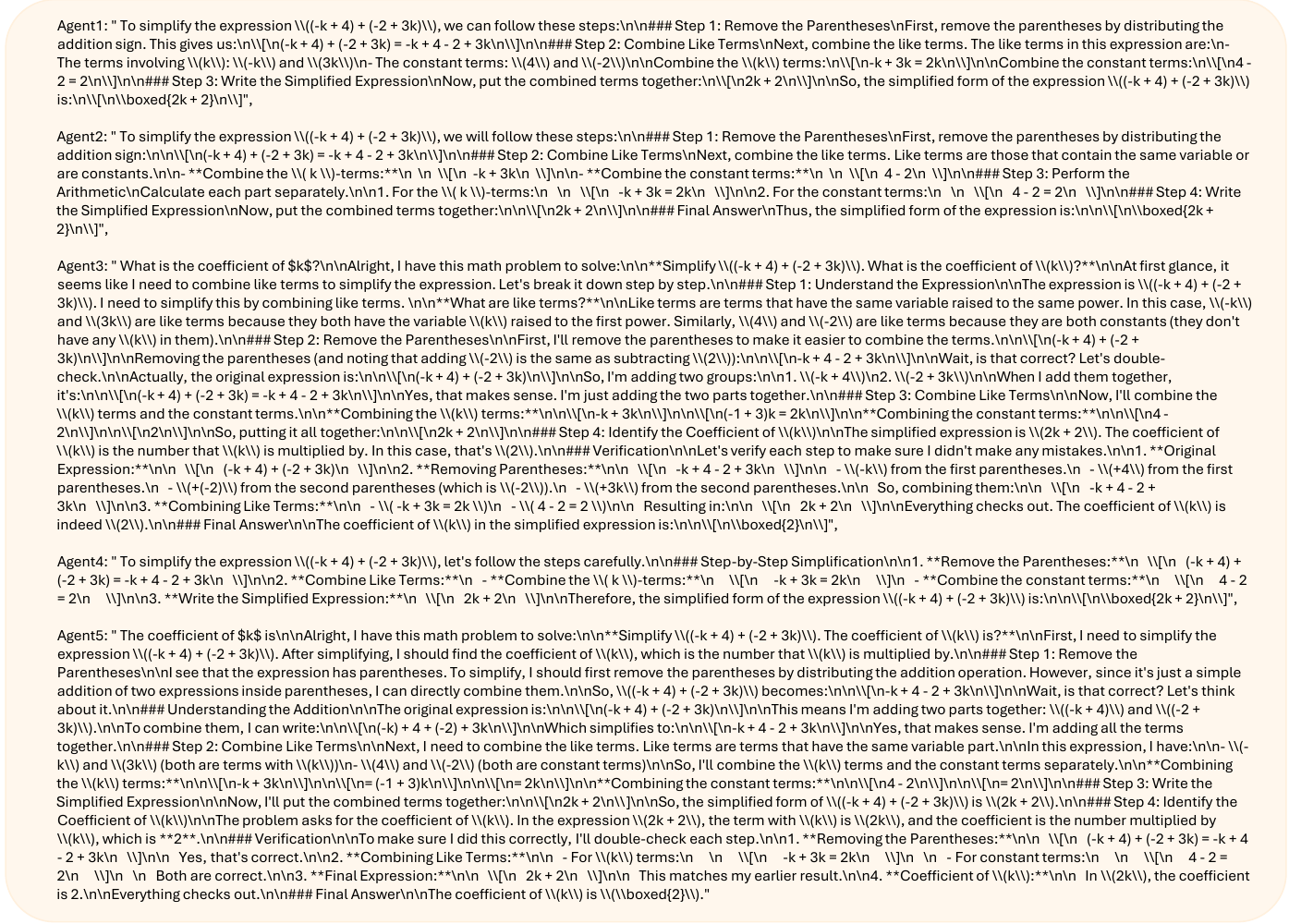}}
    \caption{
      Responses after one round debate on MATH500 with $5$ agents.
    }
    \label{fig:case2}
  \end{center}
  \vspace{-15pt}
\end{figure*}

\begin{figure*}[h]
  \begin{center}
    \centerline{\includegraphics[width=\textwidth]{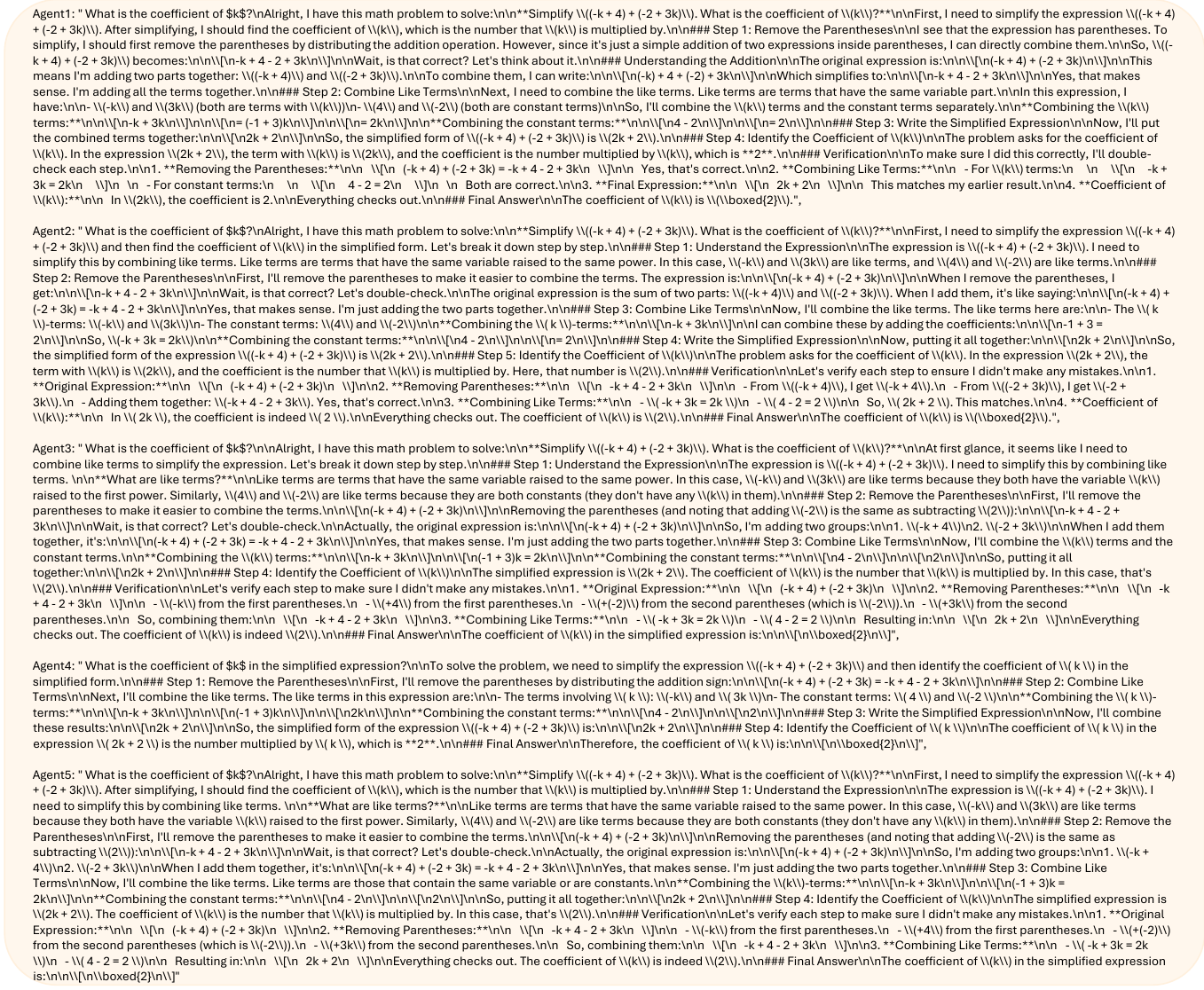}}
    \caption{
      Responses after two round debate on MATH500 with $5$ agents.
    }
    \label{fig:case3}
  \end{center}
  \vspace{-15pt}
\end{figure*}

\begin{table}[h]
  \centering
  \caption{Average response length (tokens) of correct vs.\ incorrect answers across debate rounds for the MATH500 case study.}
  \label{tab:case_length}
  \setlength{\tabcolsep}{8pt}
  \begin{tabular}{lcc}
    \toprule
    Round & Correct length & Incorrect length \\
    \midrule
    Initial (Round 0) & 2227 & 7121 \\
    Debate Round 1    & 2582 & 6189 \\
    Debate Round 2    & 3027 & 6200 \\
    Debate Round 3    & 3256 & 5951 \\
    Debate Round 4    & 3412 & 5826 \\
    \bottomrule
  \end{tabular}
\end{table}

\section{Limitations}
\label{app:limitations}

This work focuses on improving reasoning models for multi-agent debate, with the main evaluation conducted on mathematical reasoning benchmarks and additional analysis on broader reasoning tasks. We do not claim that the observed gains will directly transfer to all domains, especially open-ended generation, interactive user-facing settings, or safety-critical applications. Our experiments also use a finite set of open-source model backbones and standard MAD protocols, so future work could further evaluate SDRL under more diverse model families, deployment scenarios, and domain-specific debate settings.

\section{Broader Impact}
\label{app:broader_impact}

This work studies training and evaluation methods for reasoning language models in multi-agent debate. Potential positive impacts include improving the reliability of reasoning systems by encouraging models to compare competing solution paths and revise errors through peer feedback. Our experiments are conducted on standard reasoning benchmarks, and we do not study direct deployment in high-stakes user-facing settings.

At the same time, debate-based systems can be used in settings involving persuasion, recommendation, or decision support. If deployed without appropriate oversight, more coherent debate agents could amplify biased arguments, overstate confidence, or be used to manipulate users. We therefore believe that debate-based reasoning systems should be evaluated with domain-specific safety checks, transparent objectives, and human oversight when they are used in applications that may influence user beliefs or decisions.

\section{Additional Proofs for Section~\ref{theoretical}}
\label{app:theory_proofs}

This appendix provides detailed proofs for the results in Section~\ref{theoretical}.
Our analysis follows the DCM-based treatment in \citet{choi2025debate}, and extends it by incorporating
a private critique pseudo-count vector \(\beta\) (Definition~\ref{def:critique_update}).

\subsection{Notation and standing assumptions}
\label{app:notation}

Fix an input question \(q\) and a finite answer set \(\mathcal{A}=\{1,\dots,K\}\). Without loss of generality,
index \(1\) is the correct answer. For each agent \(i\) and round \(t\), let
\(\alpha_{i,t}\in\mathbb{R}^K_+\) be Dirichlet pseudo-counts and define the Dirichlet mean
\(\bar\theta_{i,t}:=\alpha_{i,t}/\|\alpha_{i,t}\|_1\in\Delta^K\).
The belief in the correct answer is
\[
p_{i,t} := \bar\theta^{(1)}_{i,t} = \frac{\alpha^{(1)}_{i,t}}{\|\alpha_{i,t}\|_1}.
\]

\paragraph{DCM marginal.}
Under Definition~\ref{def:dcm}, the DCM sampling scheme implies the following marginal:
for any \(k\in\{1,\dots,K\}\),
\begin{equation}
\Pr(y_{i,t}=k\mid \alpha_{i,t})=\bar\theta^{(k)}_{i,t}=\frac{\alpha^{(k)}_{i,t}}{\|\alpha_{i,t}\|_1}.
\label{eq:dcm_marginal_app}
\end{equation}

\paragraph{Neighbor counts.}
Let \(N(i)\) denote the neighbor set of agent \(i\) at round \(t\) (including itself, if desired).
Given neighbors' previous responses \(\{y_{j,t-1}\}_{j\in N(i)}\), define the count vector
\(c_{i,t}\in\mathbb{N}^K\) by
\begin{equation}
c^{(k)}_{i,t}:=\sum_{j\in N(i)}\mathbbm{1}\{y_{j,t-1}=k\}.
\label{eq:count_def_app}
\end{equation}
Because \(c_{i,t}\) counts exactly \(|N(i)|\) responses,
\begin{equation}
\|c_{i,t}\|_1=\sum_{k=1}^K c^{(k)}_{i,t} = |N(i)|.
\label{eq:count_l1_app}
\end{equation}

\paragraph{Critique-augmented update and constant mass.}
The belief update is (Definition~\ref{def:critique_update})
\begin{equation}
\alpha_{i,t}=\alpha_{i,t-1}+\beta_{i,t-1}+w_i c_{i,t},
\qquad \beta_{i,t-1}\in\mathbb{R}^K_+,\quad w_i\ge 0.
\label{eq:update_app}
\end{equation}
Throughout the proofs we use the constant-mass assumption
\begin{equation}
\|\beta_{i,t}\|_1=m_\beta,\qquad \text{for all }i,t,
\label{eq:beta_mass_app}
\end{equation}
which corresponds to adding a fixed ``budget'' of private pseudo-observations per round.

\paragraph{Why assume constant $\|\beta\|_1$.} \label{beta_explained}
We interpret \(\beta_{i,t}\) as a \emph{fixed-budget} private signal extracted from the debate context
(\emph{e.g.,} a bounded scoring head or a fixed number of self-critique samples), hence its total pseudo-count mass is controlled.
Assuming \(\|\beta_{i,t}\|_1=m_\beta\) keeps the per-round update magnitude comparable and yields closed-form drift expressions.
The analysis extends to time-varying masses by replacing \(m_\beta\) with \(\|\beta_{i,t}\|_1\), at the cost of heavier notation.

\paragraph{Filtration.} \label{filtration}
Let \(\mathcal{F}_t\) be the \(\sigma\)-algebra of information \emph{available at the start of round \(t\)}
(before sampling the round-\(t\) responses \(y_{\cdot,t}\)). Since \(\beta_{\cdot,t-1}\) is computed from the
round-\(t-1\) debate context (which includes \(y_{\cdot,t-1}\)), we include \(\beta\) only up to time \(t-1\):
\[
\mathcal{F}_t := \sigma\Big(\{\alpha_{\cdot,s}\}_{s\le t},\{y_{\cdot,s}\}_{s\le t-1},\{\beta_{\cdot,s}\}_{s\le t-1}\Big).
\]
Equivalently,
\(\mathcal{F}_{t-1}=\sigma(\{\alpha_{\cdot,s}\}_{s\le t-1},\{y_{\cdot,s}\}_{s\le t-2},\{\beta_{\cdot,s}\}_{s\le t-2})\),
so conditioning on \(\mathcal{F}_{t-1}\) treats \(\alpha_{\cdot,t-1}\) as known but keeps
\(y_{\cdot,t-1}\) (and hence \(c_{i,t}\)) random.
All conditional expectations below are taken with respect to \(\mathcal{F}_t\).

\paragraph{Critique advantage.}
Recall Definition~\ref{def:adv}:
\begin{equation}
\delta_{i,t} := \mathbb{E}\!\left[\beta^{(1)}_{i,t}\mid \mathcal{F}_t\right] - m_\beta\,p_{i,t}.
\label{eq:adv_app}
\end{equation}

\subsection{one-step drift decomposition}
\label{app:proof_drift_decomp}

\begin{lemma}
\label{lem:drift_decomp}
Under \eqref{eq:update}--\eqref{eq:pt} and \(\|\beta_{i,t-1}\|_1=m_\beta\),
\begin{equation}
\begin{aligned}
\mathbb{E}\!\left[p_{i,t}\mid \mathcal{F}_{t-1}\right]
&= p_{i,t-1}
+ \frac{\delta_{i,t-1}}{Z_{i,t-1}}\\
&\quad+ \frac{w_i|N(i)|}{Z_{i,t-1}}\,\big(\bar p_{N(i),t-1}-p_{i,t-1}\big),
\end{aligned}
\label{eq:driftdecomp}
\end{equation}
where \(C_i:=m_\beta+w_i|N(i)|\), \(Z_{i,t-1}:=\|\alpha_{i,t-1}\|_1+C_i\), and
\(\bar p_{N(i),t-1}:=\frac{1}{|N(i)|}\sum_{j\in N(i)}p_{j,t-1}\).
\end{lemma}

\begin{proof}
We prove \eqref{eq:driftdecomp} by explicitly expanding \(p_{i,t}\) and taking conditional expectation.

\paragraph{Step 1: expand the numerator and denominator.}
From \eqref{eq:update_app}, for each coordinate \(k\),
\[
\alpha^{(k)}_{i,t}=\alpha^{(k)}_{i,t-1}+\beta^{(k)}_{i,t-1}+w_i c^{(k)}_{i,t}.
\]
Therefore,
\begin{equation}
p_{i,t}
=\frac{\alpha^{(1)}_{i,t}}{\|\alpha_{i,t}\|_1}
=\frac{\alpha^{(1)}_{i,t-1}+\beta^{(1)}_{i,t-1}+w_i c^{(1)}_{i,t}}
{\|\alpha_{i,t}\|_1}.
\label{eq:p_expand_step1}
\end{equation}

\paragraph{Step 2: simplify the $\ell_1$ norm using nonnegativity.}
All terms in \eqref{eq:update_app} are componentwise nonnegative, hence \(\|\cdot\|_1\) is additive:
\[
\|\alpha_{i,t}\|_1
=\|\alpha_{i,t-1}\|_1+\|\beta_{i,t-1}\|_1+w_i\|c_{i,t}\|_1.
\]
Using \eqref{eq:beta_mass_app} and \eqref{eq:count_l1_app},
\begin{equation}
\|\alpha_{i,t}\|_1 = \|\alpha_{i,t-1}\|_1 + m_\beta + w_i|N(i)|.
\label{eq:denom_step2}
\end{equation}
Plugging \eqref{eq:denom_step2} into \eqref{eq:p_expand_step1} yields
\begin{equation}
p_{i,t}
=\frac{\alpha^{(1)}_{i,t-1}+\beta^{(1)}_{i,t-1}+w_i c^{(1)}_{i,t}}
{\|\alpha_{i,t-1}\|_1+m_\beta+w_i|N(i)|}.
\label{eq:p_expand_step2}
\end{equation}

\paragraph{Step 3: take conditional expectation given $\mathcal{F}_{t-1}$.}
Since \(\|\alpha_{i,t-1}\|_1+m_\beta+w_i|N(i)|\) is \(\mathcal{F}_{t-1}\)-measurable, we have
\begin{align}
\mathbb{E}[p_{i,t}\mid\mathcal{F}_{t-1}]
&=
\frac{\alpha^{(1)}_{i,t-1}+\mathbb{E}[\beta^{(1)}_{i,t-1}\mid\mathcal{F}_{t-1}]
+w_i\,\mathbb{E}[c^{(1)}_{i,t}\mid\mathcal{F}_{t-1}]}
{\|\alpha_{i,t-1}\|_1+m_\beta+w_i|N(i)|}.
\label{eq:condexp_step3}
\end{align}

\paragraph{Step 4: compute $\mathbb{E}[c^{(1)}_{i,t}\mid\mathcal{F}_{t-1}]$.}
By definition \eqref{eq:count_def_app},
\[
c^{(1)}_{i,t}=\sum_{j\in N(i)}\mathbbm{1}\{y_{j,t-1}=1\}.
\]
By linearity of expectation,
\begin{equation}
\mathbb{E}[c^{(1)}_{i,t}\mid\mathcal{F}_{t-1}]
=\sum_{j\in N(i)}\mathbb{E}\!\left[\mathbbm{1}\{y_{j,t-1}=1\}\mid\mathcal{F}_{t-1}\right]
=\sum_{j\in N(i)}\Pr(y_{j,t-1}=1\mid\mathcal{F}_{t-1}).
\label{eq:Ec_step4a}
\end{equation}
Under the DCM marginal \eqref{eq:dcm_marginal_app}, conditioned on \(\alpha_{j,t-1}\) we have
\[
\Pr(y_{j,t-1}=1\mid\alpha_{j,t-1})=\frac{\alpha^{(1)}_{j,t-1}}{\|\alpha_{j,t-1}\|_1}=p_{j,t-1},
\]
and since \(\alpha_{j,t-1}\) is \(\mathcal{F}_{t-1}\)-measurable, \(\Pr(y_{j,t-1}=1\mid\mathcal{F}_{t-1})=p_{j,t-1}\).
Thus \eqref{eq:Ec_step4a} becomes
\begin{equation}
\mathbb{E}[c^{(1)}_{i,t}\mid\mathcal{F}_{t-1}]
=\sum_{j\in N(i)} p_{j,t-1}
=|N(i)|\,\bar p_{N(i),t-1},
\qquad
\bar p_{N(i),t-1}:=\frac{1}{|N(i)|}\sum_{j\in N(i)}p_{j,t-1}.
\label{eq:Ec_step4b}
\end{equation}

\paragraph{Step 5: substitute and express in terms of $\delta_{i,t-1}$.}
Also note that
\begin{equation}
\alpha^{(1)}_{i,t-1}=p_{i,t-1}\|\alpha_{i,t-1}\|_1.
\label{eq:alpha1_step5}
\end{equation}
Plugging \eqref{eq:Ec_step4b} and \eqref{eq:alpha1_step5} into \eqref{eq:condexp_step3} gives
\begin{equation}
\mathbb{E}[p_{i,t}\mid\mathcal{F}_{t-1}]
=
\frac{p_{i,t-1}\|\alpha_{i,t-1}\|_1
+\mathbb{E}[\beta^{(1)}_{i,t-1}\mid\mathcal{F}_{t-1}]
+w_i|N(i)|\,\bar p_{N(i),t-1}}
{\|\alpha_{i,t-1}\|_1+m_\beta+w_i|N(i)|}.
\label{eq:condexp_step5}
\end{equation}

Define \(C_i:=m_\beta+w_i|N(i)|\) and \(Z_{i,t-1}:=\|\alpha_{i,t-1}\|_1+C_i\).
Subtract \(p_{i,t-1}\) from \eqref{eq:condexp_step5}:
\begin{align}
\mathbb{E}[p_{i,t}\mid\mathcal{F}_{t-1}] - p_{i,t-1}
&=
\frac{\mathbb{E}[\beta^{(1)}_{i,t-1}\mid\mathcal{F}_{t-1}] - m_\beta p_{i,t-1}}
{Z_{i,t-1}}
+\frac{w_i|N(i)|(\bar p_{N(i),t-1}-p_{i,t-1})}{Z_{i,t-1}}.
\label{eq:drift_split_step5}
\end{align}
By Definition~\ref{def:adv}, the first numerator is exactly \(\delta_{i,t-1}\).
Rearranging \eqref{eq:drift_split_step5} yields \eqref{eq:driftdecomp}.
\end{proof}

\subsection{Proof of Theorem~\ref{thm:drift} (critique induces drift)}
\label{app:proof_drift}

\begin{proof}
Start from Lemma~\ref{lem:drift_decomp}:
\[
\mathbb{E}[p_{i,t}\mid \mathcal{F}_{t-1}]
= p_{i,t-1}
+ \frac{\delta_{i,t-1}}{Z_{i,t-1}}
+ \frac{w_i|N(i)|}{Z_{i,t-1}}\big(\bar p_{N(i),t-1}-p_{i,t-1}\big).
\]
Under the mean-consistency condition \(\bar p_{N(i),t-1}=p_{i,t-1}\), the last term is zero.
Since \(Z_{i,t-1}=\|\alpha_{i,t-1}\|_1+C_i\), we obtain
\[
\mathbb{E}[p_{i,t}\mid \mathcal{F}_{t-1}]
=
p_{i,t-1}
+\frac{\delta_{i,t-1}}{\|\alpha_{i,t-1}\|_1+C_i},
\]
which is \eqref{eq:drift}. This completes the proof.
\end{proof}

\subsection{Proof of Lemma~\ref{lem:accum_drift} (accumulated drift and diminishing returns)}
\label{app:proof_accum}

\begin{proof}
We prove \eqref{eq:accum_drift_main} by iterating the one-step drift bound.

\paragraph{Step 1: write the one-step improvement under mean-consistency.}
By Theorem~\ref{thm:drift}, for each \(t\ge 1\),
\begin{equation}
\mathbb{E}[p_{i,t}\mid\mathcal{F}_{t-1}] = p_{i,t-1} + \frac{\delta_{i,t-1}}{\|\alpha_{i,t-1}\|_1+C_i}.
\label{eq:acc_step1}
\end{equation}
Assuming \(\delta_{i,t-1}\ge \mu\) for \(t-1=0,\dots,T-1\), we have
\begin{equation}
\mathbb{E}[p_{i,t}\mid\mathcal{F}_{t-1}] \ge p_{i,t-1} + \frac{\mu}{\|\alpha_{i,t-1}\|_1+C_i}.
\label{eq:acc_step1b}
\end{equation}

\paragraph{Step 2: remove conditioning via tower property.}
Taking expectation of \eqref{eq:acc_step1b} and using \(\mathbb{E}[\mathbb{E}[X\mid\mathcal{F}_{t-1}]]=\mathbb{E}[X]\),
\begin{equation}
\mathbb{E}[p_{i,t}] \ge \mathbb{E}[p_{i,t-1}] + \frac{\mu}{\|\alpha_{i,t-1}\|_1+C_i}.
\label{eq:acc_step2}
\end{equation}
We now make the denominator explicit.

\paragraph{Step 3: express $\|\alpha_{i,t}\|_1$ in closed form.}
By \eqref{eq:update_app} and nonnegativity,
\[
\|\alpha_{i,t}\|_1 = \|\alpha_{i,t-1}\|_1 + \|\beta_{i,t-1}\|_1 + w_i\|c_{i,t}\|_1.
\]
Using \(\|\beta_{i,t-1}\|_1=m_\beta\) and \(\|c_{i,t}\|_1=|N(i)|\),
\begin{equation}
\|\alpha_{i,t}\|_1 = \|\alpha_{i,t-1}\|_1 + (m_\beta+w_i|N(i)|) = \|\alpha_{i,t-1}\|_1 + C_i.
\label{eq:norm_rec}
\end{equation}
Iterating \eqref{eq:norm_rec} yields
\begin{equation}
\|\alpha_{i,t-1}\|_1 = \|\alpha_{i,0}\|_1 + (t-1)C_i = S_{i,0} + (t-1)C_i,
\label{eq:norm_closed}
\end{equation}
where \(S_{i,0}:=\|\alpha_{i,0}\|_1\).

\paragraph{Step 4: telescope the bound.}
Plugging \eqref{eq:norm_closed} into \eqref{eq:acc_step2} gives
\begin{equation}
\mathbb{E}[p_{i,t}] \ge \mathbb{E}[p_{i,t-1}] + \frac{\mu}{S_{i,0}+tC_i}.
\label{eq:acc_step4}
\end{equation}
Summing \eqref{eq:acc_step4} over \(t=1,\dots,T\) telescopes:
\begin{align}
\mathbb{E}[p_{i,T}]
&\ge p_{i,0} + \mu\sum_{t=1}^{T}\frac{1}{S_{i,0}+tC_i}.
\label{eq:acc_telescope}
\end{align}
This is the first inequality in \eqref{eq:accum_drift_main}.

\paragraph{Step 5: lower bound the harmonic-like sum by a logarithm.}
The function \(u\mapsto 1/(S_{i,0}+uC_i)\) is positive and decreasing.
For any decreasing \(f\), \(\sum_{t=1}^{T} f(t)\ge \int_{1}^{T+1} f(u)\,du\).
Applying this with \(f(u)=1/(S_{i,0}+uC_i)\) yields
\begin{align}
\sum_{t=1}^{T}\frac{1}{S_{i,0}+tC_i}
&\ge \int_{1}^{T+1}\frac{1}{S_{i,0}+uC_i}\,du
= \frac{1}{C_i}\log\!\left(\frac{S_{i,0}+(T+1)C_i}{S_{i,0}+C_i}\right),
\label{eq:log_bound}
\end{align}
which proves the second inequality in \eqref{eq:accum_drift_main}.
\end{proof}

\subsection{Proof of Proposition~\ref{prop:training_adv}}
\label{app:proof_prop}

\begin{proof}
The proposition is a direct implication of Theorem~\ref{thm:drift} and Lemma~\ref{lem:accum_drift}.
If training ensures \(\delta_{i,t}\ge\mu>0\) for early rounds on the evaluation-time debate distribution,
then Theorem~\ref{thm:drift} implies
\(\mathbb{E}[p_{i,t}] \ge \mathbb{E}[p_{i,t-1}] + \mu/(S_{i,0}+tC_i)\),
so \(\mathbb{E}[p_{i,t}]\) increases with \(t\) for those rounds.
Lemma~\ref{lem:accum_drift} quantifies the accumulated gain.
Finally, \citet{choi2025debate} show that even modest improvements in single-agent correctness can be amplified by voting,
under standard independence assumptions (cf.\ their Theorem~1).
\end{proof}

\subsection{A supporting lemma: plurality error implies $\ell_1$ deviation}
\label{app:lemma_L1}

Lemma~\ref{lem:corr_neff} uses the following deterministic fact, which appears as Lemma~2 in \citet{choi2025debate}.
We reproduce it here for completeness.

\begin{lemma}[Plurality error implies $\ell_1$ deviation]
\label{lem:mv_implies_l1}
Let \(p\in\Delta^K\) satisfy \(p_1>p_2\ge \cdots \ge p_K\) and define \(\Delta:=p_1-p_2>0\).
Let \(\hat p\in\Delta^K\) be any empirical distribution on \(K\) classes and let
\(y_{\mathrm{mv}}:=\arg\max_k \hat p_k\) denote the plurality vote.
If \(y_{\mathrm{mv}}\neq 1\), then
\begin{equation}
\|\hat p-p\|_1 \ge \Delta.
\label{eq:l1_gap}
\end{equation}
\end{lemma}

\begin{proof}
Assume \(y_{\mathrm{mv}}\neq 1\). Then there exists an index \(j\neq 1\) such that
\(\hat p_j\ge \hat p_1\). Consider the \(\ell_1\) distance:
\[
\|\hat p-p\|_1 = \sum_{k=1}^K |\hat p_k-p_k| \ge |\hat p_1-p_1| + |\hat p_j-p_j|.
\]
Using the inequality \(|a|+|b|\ge |a-b|\), we obtain
\begin{align}
|\hat p_1-p_1| + |\hat p_j-p_j|
&\ge |(\hat p_1-p_1) - (\hat p_j-p_j)|
= |(\hat p_1-\hat p_j) - (p_1-p_j)|.
\label{eq:ab_ineq}
\end{align}
Since \(\hat p_j\ge \hat p_1\), we have \(\hat p_1-\hat p_j\le 0\), hence
\[
(\hat p_1-\hat p_j) - (p_1-p_j) \le -(p_1-p_j).
\]
Taking absolute values yields
\[
|(\hat p_1-\hat p_j) - (p_1-p_j)| \ge p_1-p_j.
\]
Finally, because \(p_1>p_2\ge p_j\), we have \(p_1-p_j\ge p_1-p_2=\Delta\).
Combining the inequalities implies \(\|\hat p-p\|_1\ge \Delta\), proving \eqref{eq:l1_gap}.
\end{proof}

\subsection{Proof of Lemma~\ref{lem:corr_neff} (correlation shrinks effective ensemble size)}
\label{app:proof_corr}

\begin{proof}
We expand the proof sketch into explicit steps.

\paragraph{Step 1: reduce plurality error to an $\ell_1$ deviation event.}
By Lemma~\ref{lem:mv_implies_l1}, the error event implies
\[
\{y_{\mathrm{mv}}\neq 1\} \subseteq \{\|\hat p-p\|_1 \ge \Delta\}.
\]
Therefore,
\begin{equation}
\Pr(y_{\mathrm{mv}}\neq 1) \le \Pr(\|\hat p-p\|_1 \ge \Delta).
\label{eq:step1_bound}
\end{equation}

\paragraph{Step 2: convert the $\ell_1$ deviation to an $\ell_2$ deviation.}
For any vector \(v\in\mathbb{R}^K\), \(\|v\|_1\le \sqrt{K}\|v\|_2\).
Hence,
\[
\{\|\hat p-p\|_1 \ge \Delta\} \subseteq \{\|\hat p-p\|_2 \ge \Delta/\sqrt{K}\}.
\]
Thus,
\begin{equation}
\Pr(\|\hat p-p\|_1 \ge \Delta) \le \Pr(\|\hat p-p\|_2 \ge \Delta/\sqrt{K}).
\label{eq:step2_bound}
\end{equation}

\paragraph{Step 3: apply Markov's inequality to the squared $\ell_2$ norm.}
Since \(\|\hat p-p\|_2^2\ge 0\), Markov's inequality gives
\[
\Pr(\|\hat p-p\|_2^2 \ge (\Delta^2/K)) \le \frac{\mathbb{E}\|\hat p-p\|_2^2}{\Delta^2/K}
= \frac{K\,\mathbb{E}\|\hat p-p\|_2^2}{\Delta^2}.
\]
Equivalently,
\begin{equation}
\Pr(\|\hat p-p\|_2 \ge \Delta/\sqrt{K})
\le \frac{K\,\mathbb{E}\|\hat p-p\|_2^2}{\Delta^2}.
\label{eq:step3_markov}
\end{equation}

\paragraph{Step 4: express $\mathbb{E}\|\hat p-p\|_2^2$ as a sum of variances.}
By definition,
\[
\|\hat p-p\|_2^2 = \sum_{k=1}^K (\hat p_k-p_k)^2.
\]
Taking expectation and using \(\mathbb{E}[\hat p_k]=p_k\) (since each \(\hat p_k\) is an empirical mean) gives
\begin{equation}
\mathbb{E}\|\hat p-p\|_2^2
= \sum_{k=1}^K \mathbb{E}[(\hat p_k-p_k)^2]
= \sum_{k=1}^K \mathrm{Var}(\hat p_k).
\label{eq:step4_varsum}
\end{equation}

\paragraph{Step 5: bound each $\mathrm{Var}(\hat p_k)$ using the correlation parameter.}
Define indicator variables \(I^{(k)}_n:=\mathbbm{1}\{Y_n=k\}\).
Then
\[
\hat p_k = \frac{1}{N}\sum_{n=1}^N I^{(k)}_n.
\]
Therefore,
\begin{align}
\mathrm{Var}(\hat p_k)
&= \mathrm{Var}\!\left(\frac{1}{N}\sum_{n=1}^N I^{(k)}_n\right)
= \frac{1}{N^2}\mathrm{Var}\!\left(\sum_{n=1}^N I^{(k)}_n\right)\nonumber\\
&= \frac{1}{N^2}\left(\sum_{n=1}^N \mathrm{Var}(I^{(k)}_n) + \sum_{a\neq b}\mathrm{Cov}(I^{(k)}_a,I^{(k)}_b)\right).
\label{eq:var_expand}
\end{align}
Since \(\mathbb{E}[I^{(k)}_n]=p_k\), we have \(\mathrm{Var}(I^{(k)}_n)=p_k(1-p_k)\).
By the definition of \(\rho\) in Lemma~\ref{lem:corr_neff}, for any \(a\neq b\),
\(\mathrm{Cov}(I^{(k)}_a,I^{(k)}_b)\le \rho\, p_k(1-p_k)\).
Plugging these into \eqref{eq:var_expand} yields
\begin{align}
\mathrm{Var}(\hat p_k)
&\le \frac{1}{N^2}\left(N\,p_k(1-p_k) + N(N-1)\rho\,p_k(1-p_k)\right)
= \frac{p_k(1-p_k)}{N}\left(1+(N-1)\rho\right).
\label{eq:var_bound_k}
\end{align}

\paragraph{Step 6: sum over $k$ and finish.}
Summing \eqref{eq:var_bound_k} over \(k\) and using \(\sum_{k=1}^K p_k(1-p_k)=1-\sum_{k=1}^K p_k^2\le 1\),
\begin{equation}
\sum_{k=1}^K \mathrm{Var}(\hat p_k)
\le \frac{1+(N-1)\rho}{N}.
\label{eq:var_sum_bound}
\end{equation}
Combining \eqref{eq:step1_bound}, \eqref{eq:step2_bound}, \eqref{eq:step3_markov}, \eqref{eq:step4_varsum}, and \eqref{eq:var_sum_bound} gives
\[
\Pr(y_{\mathrm{mv}}\neq 1)
\le \frac{K}{\Delta^2}\cdot \frac{1+(N-1)\rho}{N}
= \frac{K(1+(N-1)\rho)}{N\Delta^2}.
\]
Defining \(N_{\mathrm{eff}}:=N/(1+(N-1)\rho)\) yields the statement in Lemma~\ref{lem:corr_neff}.
\end{proof}

\paragraph{Remark (independence vs.\ correlation).}
When \(\rho=0\) (independent agents), Lemma~\ref{lem:corr_neff} reduces to a polynomial tail bound
\(\Pr(y_{\mathrm{mv}}\neq 1)\le K/(N\Delta^2)\), which is generally looser than the
exponential bound in \citet{choi2025debate} (their Theorem~1).
The advantage of Lemma~\ref{lem:corr_neff} is that it makes the dependence on correlation explicit
through the effective size \(N_{\mathrm{eff}}\), which is useful for analyzing multi-round debates where
agent outputs typically become more correlated as the context grows.

\subsection{Additional discussion: linking theory to peak-then-decline behavior}
\label{app:discussion_peak}

Lemma~\ref{lem:accum_drift} shows that, under sustained positive advantage, the improvement in \(p_{i,t}\)
has diminishing returns: the gain from round \(t\) scales as \(1/(S_{i,0}+tC_i)\).
Lemma~\ref{lem:corr_neff} shows that the benefit of plurality voting depends on an effective sample size
\(N_{\mathrm{eff}}(t)\) that can shrink as correlations increase.
In multi-round debate, it is common for both phenomena to occur simultaneously:
early rounds exhibit both positive critique advantage and low correlation,
while later rounds exhibit smaller per-round drift and higher correlation.
This provides a simple theoretical explanation for the empirical ``rise-then-fall'' pattern
often observed in multi-round debate accuracy.

\end{document}